%% file: main.tex
\documentclass{article}

\usepackage[preprint, nonatbib]{neurips_2026}

\usepackage[utf8]{inputenc} 
\usepackage[T1]{fontenc}    
\usepackage[textsize=tiny]{todonotes}
\usepackage[dvipsnames]{xcolor}
\usepackage{hyperref}       
\usepackage{url}            
\usepackage{booktabs}       
\usepackage{amsfonts}       
\usepackage{nicefrac}       
\usepackage{microtype}      
\usepackage{amsmath}
\usepackage{amssymb}
\usepackage{mathtools}
\usepackage{amsthm}
\usepackage{xspace}
\usepackage{graphicx}
\usepackage{hyperref}
\usepackage{pifont}
\usepackage{bbm}
\usepackage{multirow}
\usepackage{makecell}
\usepackage{enumitem}
\usepackage{wrapfig}
\usepackage{subcaption}
\usepackage{placeins}
\usepackage{algorithm}
\usepackage{colortbl}
\usepackage[square,numbers]{natbib}

\usepackage[most]{tcolorbox}

\definecolor{darkpurple}{HTML}{660066}
\definecolor{figred}{HTML}{B85450}
\definecolor{figgreen}{HTML}{82B366}
\definecolor{figgray}{HTML}{808080}

\newcommand{\cmark}{\textcolor{green!60!black}{\ding{51}}}
\newcommand{\xmark}{\textcolor{red!70!black}{\ding{55}}}

\newcommand*\circledred[1]{%
\tikz[baseline=(char.base)]{
  \node[shape=circle, draw=BrickRed!60, fill=BrickRed!10, thick, inner sep=1pt] (char) {\scriptsize\textsf{#1}};
}}
\newcommand*\circledblue[1]{%
\tikz[baseline=(char.base)]{
  \node[shape=circle, draw=NavyBlue!60, fill=NavyBlue!10, thick, inner sep=1pt] (char) {\scriptsize\textsf{#1}};
}}
\usepackage{algpseudocode}
\algrenewcommand\algorithmicforall{\textbf{for all}}
\makeatletter
\renewcommand{\ALG@step}{\addtocounter{ALG@line}{1}\arabic{ALG@line}}
\makeatother

\newcommand{\tmax}{{t_{\mathrm{max}}}}
\algnewcommand{\Input}{\item[\textbf{Input:}]}

\theoremstyle{plain}
\newtheorem{theorem}{Theorem}[section]
\newtheorem{proposition}[theorem]{Proposition}
\newtheorem{lemma}[theorem]{Lemma}
\newtheorem{corollary}[theorem]{Corollary}
\theoremstyle{definition}
\newtheorem{definition}[theorem]{Definition}
\newtheorem{assumption}[theorem]{Assumption}
\theoremstyle{remark}
\newtheorem{remark}[theorem]{Remark}

\definecolor{darkpurple}{HTML}{660066}
\definecolor{figred}{HTML}{B85450}
\definecolor{figgreen}{HTML}{82B366}
\definecolor{figgray}{HTML}{808080}
\definecolor{darkgreen}{rgb}{0.0, 0.5, 0.0}
\definecolor{lightyellow}{HTML}{FFE699}
\definecolor{red_revision}{HTML}{FF0000}
\definecolor{darkblue}{HTML}{2E6EB3}
\definecolor{derkgreen}{HTML}{3E7D00}
\definecolor{darkyellow}{HTML}{D99542}
\definecolor{darkpurple}{HTML}{660066}

\input{math_commands.tex}

\title{Adaptive Experimentation for Censored Survival Outcomes}

\author{%
  Yuxin Wang\thanks{Correspondence to: \texttt{Yuxin.Wang1@lmu.de}},\, Dennis Frauen, Jonas Schweisthal, Maresa Schröder, Emil Javurek,\\
  \textbf{Stefan Feuerriegel}\\
  LMU Munich \\
  Munich Center of Machine Learning (MCML)
}

\begin{document}

\maketitle

\begin{abstract}
Adaptive experimentation enables efficient estimation of causal effects, but existing methods are not designed for survival data with censoring, where event times are only partially observed (e.g., overall survival in cancer trials but with dropout). In this paper, we develop a novel framework for adaptive experimentation to estimate causal effects under right censoring. For this, we derive the semiparametric efficiency bound for the average survival effect curve as a function of the treatment allocation policy and thereby obtain a closed-form efficiency-optimal allocation policy. The policy generalizes classical Neyman allocation to survival settings by prioritizing patient strata where both event and censoring dynamics induce high uncertainty. Building on this, we propose the \textit{Adaptive Survival Estimator (ASE)}, an adaptive framework that learns the allocation policy and estimates the average survival effect curve sequentially. Our framework has three main benefits: (i)~it accommodates arbitrary machine learning models for nuisance estimation; (ii)~it is guided by a closed-form efficiency-optimal allocation policy; and (iii)~it admits strong theoretical guarantees, including asymptotic normality via a martingale central limit theorem. We demonstrate our framework across various numerical experiments to show consistent efficiency gains over uniform randomization and censoring-agnostic baselines. 
\end{abstract}

\section{Introduction}

Randomized controlled trials (RCTs) are the gold standard for causal inference~\citep{imbens.2015, wager.2024}, but RCTs have important limitations. By design, RCTs rely on a static rule to allocate patients to treatment arms that is fixed before data collection, and, as a result, RCTs ignore information that becomes available during the trial and may allocate patients inefficiently. This is particularly problematic in survival trials, where patient dropout and administrative censoring are common; this causes event times to be only partially observed, which further complicates estimation.

An alternative to RCTs is \emph{\textbf{adaptive experimentation}}, which updates how patients are assigned to treatment arms as data are collected~\citep{robertson.2023, vanderlaan.2008, zhang.2025}. Here, outcomes from previous patients are used to inform how future patients are assigned to treatments. The idea of adaptive randomization is particularly relevant in settings where population sizes are limited by costly recruitment, ethical concerns, or rare diseases~\citep{berry.2011, guidance.2010, vanamsterdam.2026}. 

Prominent examples are response-adaptive randomization (RAR) and the covariate-adjusted variant (CRAR), which adjust the treatment assignment probabilities based on observed outcomes and patient characteristics~\citep[e.g.,][]{chow.2008,hu.2006,rosenberger.2001, thompson.1933, zhang.2007}. These methods are primarily designed to improve patient welfare by allocating more patients to better-performing treatments during the trial. A separate line of work studies adaptive designs that aim to minimize asymptotic variance of the ATE estimator, thus leading to an allocation policy based on outcome variability, such as Neyman allocation~\citep[e.g.,][]{hahn.2011, kato.2024, cook.2024, dai.2023}. However, both paradigms are developed for \textit{fully observed outcomes}, and are \textbf{\underline{not}}  directly applicable to survival data with censoring, where event times are only \textit{partially observed} (e.g., overall survival time in cancer trials \citep{klein.2003a,wiegrebe.2024a}).

\begin{tcolorbox}[
    colback=NavyBlue!5!white,
    colframe=NavyBlue!75!black,
    arc=1mm,
    left=4.5pt, right=4.5pt, top=3pt, bottom=3pt
]
\textbf{Motivating example.} GBM AGILE~\citep{cloughesy.2022} is an adaptive clinical trial for glioblastoma (an aggressive form of brain cancer) that compares different treatments based on overall survival time. In cancer care, outcomes are time-to-event and often censored, as patients may not experience the event during the study or may drop out (e.g., due to changes in therapies or loss to follow-up)~\citep{klein.2003a}.

\medskip 
However, in such settings, adaptive trials may update treatment allocation based on observed survival outcomes, while not explicitly accounting for how censoring affects information accumulation. As a result, treatment assignment may overlook censoring-induced uncertainty, leading to inefficient allocation. For example, if patients in one treatment group drop out more frequently, more observations are censored in that group, making it harder to precisely estimate the overall treatment effect.
\end{tcolorbox}

\begin{figure}
    \centering
    \includegraphics[width=0.85\linewidth]{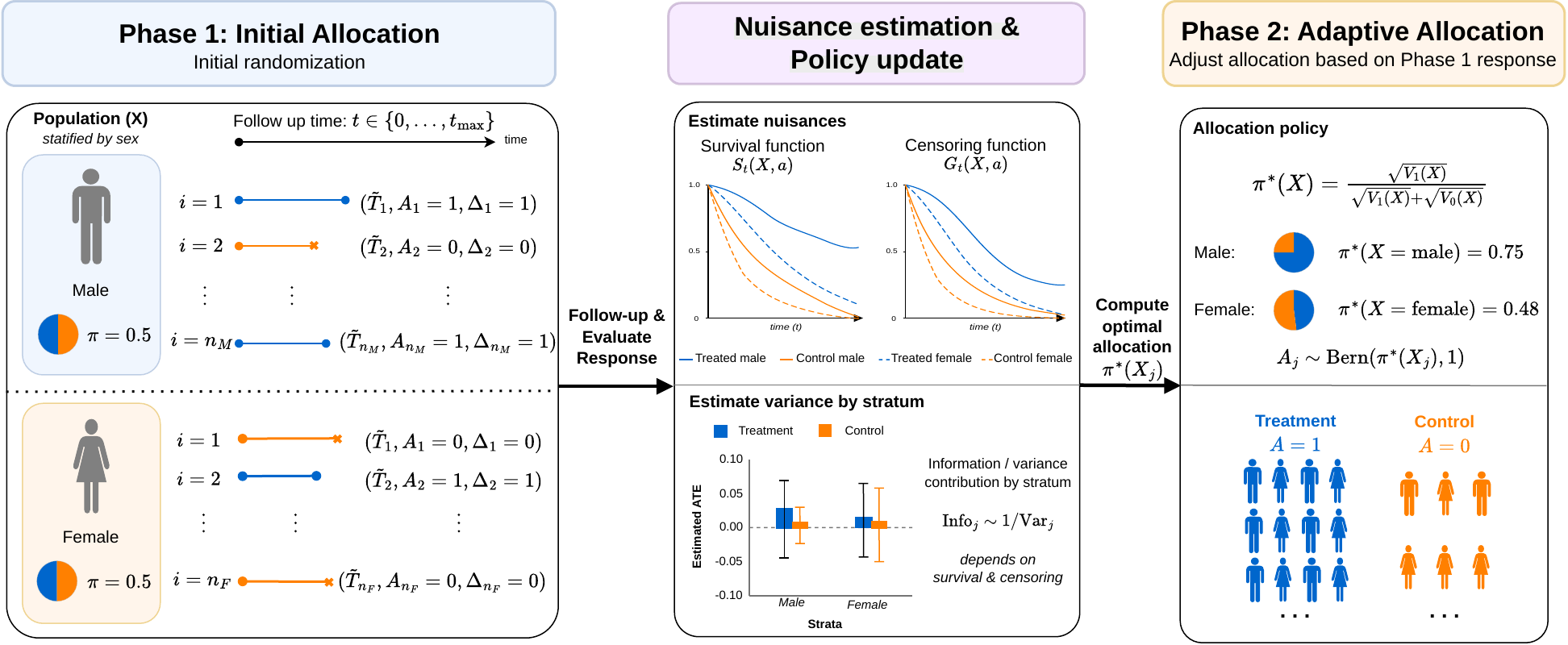}
    \caption{\textbf{Our semi-parametric framework for adaptive experimentation with censored survival outcomes.}}
    \label{fig:1}
    \vspace{-0.6cm}
\end{figure}

Existing methods for adaptive experimentation~\citep[e.g.,][]{hahn.2011, kato.2024, thompson.1933, vanderlaan.2008} have two main \textbf{shortcomings} in survival settings: \circledred{1} \emph{Optimal allocation policy}: Standard methods for adaptive experimentation account for uncertainty in observed outcomes, but survival data introduce \textit{additional} uncertainty due to \textit{censoring}. $\Rightarrow$ As a result, an efficient allocation policy should prioritize an arm not only based on variability in outcomes, but also based on the level of censoring, which may vary across treatments and patient subgroups.
\circledred{2} \emph{Estimation}: Existing methods are designed for settings with fully observed outcomes. In survival settings, outcomes are only partially observed, which makes inference more challenging. $\Rightarrow$ Therefore, it requires an estimator for the average survival effect curve that remains valid as data accumulate across rounds under adaptive allocation.

\textbf{In this paper,} we address the above challenges and propose \textit{a novel adaptive experimentation framework for censored survival outcomes} (see Figure~\ref{fig:1}):
\vspace{-0.2cm}
\begin{enumerate}[leftmargin=0.5cm, itemsep=0pt]
\item[\circledred{1}] \textbf{Optimal allocation under censoring.} We derive the semiparametric efficiency bound for the average survival effect curve as a function of the treatment allocation policy. This allows us to characterize the variance structure induced by both the survival process and censoring. We thereby obtain a \emph{closed-form}, efficiency-optimal allocation policy that is tailored to survival settings.  
\item[\circledred{2}] \textbf{Adaptive Survival Estimator (ASE).}  We propose the \emph{\textbf{A}daptive \textbf{S}urvival \textbf{E}stimator (ASE)}, a sequential procedure that learns the optimal allocation policy and estimates the average survival effect curve. Specifically, ASE uses flexible machine learning models to estimate nuisance components, then updates the probability of assigning newly enrolled patients to different treatment arms, and offers valid inference for the average survival effect curve. \vspace{-0.2cm}
\end{enumerate}

Our \textbf{contributions}\footnote{Code available at: \url{https://anonymous.4open.science/r/adaptive_censoring-C872}} are: (1)~We develop a novel framework for adaptive experimentation with censored survival outcomes. Thereby, we generalize the Neyman allocation by deriving an allocation policy that minimizes the trace of the semiparametric efficiency bound for the average survival effect curve. (2)~We introduce the \emph{Adaptive Survival Estimator (ASE)}, a new sequential approach for updating treatment allocation policy and estimating the average survival effect curve. We provide strong theoretical guarantees, including asymptotic normality of ASE, and time-uniform asymptotic confidence sequences for valid inference at arbitrary stopping times. (3) We evaluate our ASE across various numerical experiments and observe consistent efficiency gains from our adaptive allocation policy over uniform randomization and censoring-agnostic baselines.

\vspace{-0.2cm}
\section{Related work\protect\footnote{We provide an extended literature review in Appendix~\ref{app:extended_related_work}.}}
\vspace{-0.2cm}

\textbf{Response-adaptive randomization:} 
Response-adaptive randomization methods are primarily designed to improve patient outcomes during a trial by allocating more patients to treatment arms that appear more effective. $\bullet$\,\textit{Response-adaptive randomization (RAR)} traces back to the foundational contribution of~\citet{thompson.1933} with Bayesian extensions that drive allocation via posterior estimates of treatment effects~\citep{berry.2010, robertson.2023}. Frequentist formulations update the probabilities of assigning patients to treatment arms as data accumulate during a trial, so that the allocation to treatment can adapt to observed outcomes~\citep{ chow.2008,hu.2006}. $\bullet$\,\textit{Covariate-adjusted response-adaptive randomization (CRAR)} extends RAR by incorporating baseline covariates into the allocation rule, which thus yields individualized assignment probabilities that can be used to improve patient welfare~\citep{mukherjee.2025, rosenberger.2001, rosenberger.2008, zhang.2007}. $\Rightarrow$ \textit{However, these methods are primarily developed for welfare optimization rather than for improving estimation efficiency.}

\textbf{Adaptive experimentation for efficient ATE estimation:} Another line of work studies adaptive designs that aim to minimize the variance of treatment effect estimation~\citep{vanderlaan.2008}. In these settings, optimal allocation is given by \emph{Neyman allocation}, which assigns more samples to treatment arms with higher outcome variability. Existing methods can be grouped by the experiment design: (i)~in two-staged designs, one can asymptotically attain the semiparametric efficiency bound~\citep{hahn.2011}, while (ii)~fully sequential designs such as A2IPW achieve variance-optimal Neyman allocation~\citep{cook.2024, dai.2023, kato.2024}. There are also several extensions (see Appendix~\ref{app:extended_related_work}). $\Rightarrow$ \emph{However, these methods assume \textbf{fully observed outcomes} and apply Neyman allocation based only on outcome variance but ignore the censoring-induced uncertainty.}

\textbf{Semiparametric estimation of survival treatment effects:} We build on semiparametric efficiency theory for treatment effects with time-to-event outcomes, but our focus is different: we leverage the concepts to improve \emph{adaptive} experimental design rather than estimation under a \emph{fixed} treatment assignment. This literature stream provides efficient estimators under censoring \citep{cai.2020, mao.2018,rubin.2007, schaubel.2011, westling.2024,vanderlaan.2003}. There are also several extensions (see Appendix~\ref{app:extended_related_work}). $\Rightarrow$ \emph{However, this line of work focuses on estimation under a fixed treatment assignment mechanism and does not address adaptive design. While it characterizes the semiparametric efficiency bound, it does not leverage it to guide treatment allocation. In contrast, our work uses the efficient influence function to derive sequential allocation policies that are optimal under censoring.}

\textbf{Research gap:}  In sum, \underline{no} existing work jointly derives an efficiency-optimal allocation policy including a valid estimator for survival outcomes under censoring in adaptive experiments (see Table~\ref{tab:related_work} in the Appendix).

\section{Problem setup}\label{sec:setup}
\vspace{-0.2cm}

\textbf{Setting.} We consider the problem of estimating the average treatment effect (ATE) from time-to-event data \cite{cui.2023, curth.2021, frauen.2025}.\footnote{We deal with the problem of right censoring, which is common in survival analysis settings.} That is, we consider a population $(X,A,T,C)\sim\mathbb P$, where $X\in \gX \subseteq \mathbb{R}^p$ are observed covariates, $A\in \gA = \{0,1\}$ is a binary treatment, $T\in\gT$ is the event time of interest (e.g., death of the patients), and $C\in \gT$ is the censoring time (e.g., dropout time). We further use $\Delta = \mathbf{1}(T\leq C)$ as the event indicator\footnote{By defining $\Delta = \mathbf{1}(T \leq C)$, we allow for ties between $T$ and $C$, which is reflected in the definition of $G_{t-1}$. We extend our results to the no-ties case, i.e., $\Delta = \mathbf{1}(T < C)$, in Appendix~\ref{app:no-ties}.}, and $\tilde{T} = \min\{T, C\}$. Let $\gO = (X, A, \tilde{T}, \Delta)$ denote the observed population. Throughout, we assume a discrete-time setting $(T, C) \in \mathcal T=\{0,\dots,\tmax\}$ and all units share the same discrete time horizon across experiment rounds.

The experiment proceeds over $R \in \mathbb{N}$ rounds. At each round $r$, a new patient with covariates $X_r$ is i.i.d. drawn from $\gX$. The experimenter observes $X_r$ and selects a treatment $A_r \sim \pi_r(\cdot \mid X_r, \gH_{r-1})$, where $\pi_r$ is an \textbf{adaptive treatment allocation policy} that depends on the current covariates $X_r$ and past observations
\begin{align}
\gH_{r-1} = \left\{(x_1,a_1,\tilde{t}_1,\delta_1), \dots, (x_{r-1},a_{r-1}, \tilde{t}_{r-1},\delta_{r-1}) \right\}.
\end{align}
Following the treatment assignment $a_r$, one observes the outcome $\tilde{t}_r$ and the survival indicator $\delta_r$. The full observation at round $r$ is thus $(x_{r},a_{r}, \tilde{t}_{r},\delta_{r})$. After $R$ rounds, the experimenter estimates the ATE from accumulated data $\gH_{R}=\{(x_{r},a_{r}, \tilde{t}_{r},\delta_{r})\}_{r=1}^R$. As a result, we allow the treatment assignment policy to evolve over rounds based on accumulated data. 

\textbf{Key definitions:} For $t\in\mathcal T$, we define the \textit{(conditional) survival and censoring functions}
$S_t(x,a)=\mathbb P(T>t\mid X=x,A=a)$, $G_t(x,a)=\mathbb P(C>t\mid X=x,A=a)$,
with the convention $S_{-1}(x,a) = G_{-1}(x,a)\equiv 1$. We also define the observed \textit{discrete-time hazards}
$\lambda_t^S(x,a)=\mathbb P(\widetilde T=t,\Delta=1\mid \widetilde T\ge t, X=x,A=a)$, and $\lambda_t^G(x,a)=\mathbb P(\widetilde T=t,\Delta=0\mid \widetilde T\ge t, X=x,A=a)$.

\textbf{Estimand:} We use the potential outcomes framework~\citep{rubin.1974} to formalize our causal inference task. Let $T(a)\in \gT$ denote the potential event time under treatment $a\in\{0,1\}$. For a fixed $t \in \gT$, we are interested in the \textit{average survival effect curve} 
\begin{align}
\tau_t=\E[\tau_t(X)]\text{, where }
\tau_t(x)=\mathbb P(T(1)>t\mid X=x)-\mathbb P(T(0)>t\mid X=x).
\end{align}

\vspace{-0.3cm}
We define the nuisance vector $\eta_t(X):=\bigl(\lambda^S_i(X,1),\lambda^S_i(X,0),\lambda^G_{i-1}(X,1),\lambda^G_{i-1}(X,0)\bigr)_{i=0}^t$.

\vspace{-0.1cm}
\textbf{Identifiability}: We impose the following standard assumptions to ensure the identifiability of $\tau_t$.
\begin{assumption}[Standard causal inference assumptions]
\label{ass:causal}
For all $a\in\{0,1\}$ and $x\in\mathcal X$:
(i) \emph{consistency}: $T(a)=T$ whenever $A=a$;
(ii) \emph{treatment overlap}: $0<\pi(x)<1$ whenever $\mathbb P(X=x)>0$;
(iii) \emph{ignorability}: $A\perp (T(0),T(1))\mid X=x$.
\end{assumption}

\begin{assumption}[Survival-specific assumptions]
\label{ass:survival}
For all $a\in\{0,1\}$ and $x\in\mathcal X$ with $\mathbb P(X=x,A=a)>0$:
(i) \emph{censoring overlap}: $G_{t-1}(x,a)>0$;
(ii) \emph{survival overlap}: $S_{t-1}(x, a)>0$;
(iii) \emph{non-informative censoring}: $T\perp C\mid X=x,A=a$.
\end{assumption}
\vspace{-0.2cm}

Assumption~\ref{ass:causal} is standard in causal inference~\citep{imbens.2004, robins.2000, laan.2006}: (i)~consistency ensures that a patient's observed outcome under treatment $a$ equals their potential outcome $T(a)$ and rules out interference between patients; (ii)~treatment overlap means that every patient has a positive probability of receiving each treatment, which thereby ensures sufficient support in the data; and (iii)~ignorability means there are no unobserved confounders that can bias our estimation. Assumption~\ref{ass:survival} is commonly imposed for survival analysis~\citep{vanderlaan.2003, cui.2023} and ensures that we have sufficient uncensored and surviving individuals for each covariate value, and the censoring mechanism is independent of a patient's survival time given covariates and treatment. Under Assumptions~\ref{ass:causal} and \ref{ass:survival}, one can identify $\tau_t$ via~\citep{vanderlaan.2003}:

\vspace{-0.2cm}
\begin{align}
\label{eq:plugin}
    \tau_t = \E \left[ S_t(X,1) - S_t(X, 0)\right] = \E \left[ \prod_{i=0}^t (1-\lambda^S_i(X, 1))- \prod_{i=0}^t (1-\lambda^S_i(X, 0))\right].
\end{align}

\vspace{-0.3cm}
\textbf{Efficient influence function (EIF) in the non-adaptive setting.} Cui et al. \cite{cui.2023} have derived the non-centered EIF $\phi_t$ for the ATE estimator in Eq.~\ref{eq:eif} in the \emph{non-adaptive} setting, where the treatment assignment policy is \emph{fixed} over time:
\begin{align}
\label{eq:eif}
\phi_t(\gO;\tau_t,\eta_t) =
S_t(X,1)-S_t(X,0)
-\frac{(A-\pi(X))\,\xi(\gO, \eta_t)\,S_t(X,A)}{\pi(X)(1-\pi(X))},
\end{align}

\vspace{-0.4cm}
where $\xi(\gO,\eta_t) = \sum_{i=0}^t
\frac{\mathbf 1(\widetilde T=i,\Delta=1)-\mathbf 1(\widetilde T\ge i)\,\lambda_i^S(X,A)}{S_i(X,A)\,G_{i-1}(X,A)}$ and $G_{t-1}(X,A) = \prod_{i=0}^{t-1}1-\frac{\lambda_i^G(x,a)}{1 - \lambda_i^S(x,a)}$. 

\vspace{-0.2cm}
\section{Semiparametric efficiency bound and optimal treatment allocation}
\label{sec:optimal_allocation}
\vspace{-0.2cm}

To guide optimal experiment design in survival settings, we first derive the semiparametric efficiency bound for ATE estimation. This then characterizes the variance-minimizing allocation strategy and motivates our ASE method later (see §~\ref{sec:ASE_full}).

\textbf{Semiparametric efficiency bound:} We consider a setting where the experimenter has direct control over $\pi$ and can update the allocation policy to improve estimation efficiency. Hence, we are interested in the problem of \emph{designing} $\pi$ to minimize the asymptotic variance of the ATE estimator.

\begin{theorem}[Semiparametric efficiency bound]
\label{thm:eff_bound}
Under Assumptions~\ref{ass:causal} and~\ref{ass:survival}, the semiparametric efficiency bound for estimating $\tau = (\tau_0, \ldots, \tau_{t_{\max}})^\top$ is $\Sigma_{\mathrm{eff}}(\pi) \in \mathbb{R}^{(t_{\max}+1)\times(t_{\max}+1)}$,
\begin{align}
\Sigma_{\mathrm{eff}}(\pi) = \mathbb{E}\!\left[\frac{\Sigma_1(X)}{\pi(X)} + \frac{\Sigma_0(X)}{1-\pi(X)}\right] + \mathbb{E}[b(X)b(X)^\top],
\end{align}

\vspace{-0.3cm}
where {\small$\Sigma_a(X) = \Var(\bm{S}(X,a) \odot \bm{\xi}(\gO;\eta_{\tmax})\mid X, A=a)$}, and {\small$b(X) := \bm{S}(X,1) - \bm{S}(X,0) - \boldsymbol{\tau}$}.
\end{theorem}

\vspace{-0.3cm}
\begin{proof}
See Appendix~\ref{app:proof_eff_bound}, we provide a detailed summary of notation in Appendix~\ref{app:notation}.
\end{proof}
\vspace{-0.3cm}

\textbf{Optimal treatment allocation.} Given the efficiency bound $\Sigma_{\mathrm{eff}}(\pi)$ from Theorem~\ref{thm:eff_bound}, we now derive the treatment allocation policy $\pi^\star$ that minimizes it. Since $\Sigma_{\mathrm{eff}}(\pi)$ is a matrix, an optimal $\pi^\star$ should minimize a scalar summary of $\Sigma_{\mathrm{eff}}(\pi)$ over all feasible allocation policies $\Pi_\alpha = \{\pi : \alpha \leq \pi(x) \leq 1-\alpha \text{ for all } x\}$.

\begin{definition}
\itshape
    A treatment allocation policy $\pi^*$ is A-optimal if $\pi^* \in \arg \min \mathrm{tr}(\Sigma_{\mathrm{eff}}(\pi))$, where $\mathrm{tr}$ is the trace operator. That is, an A-optimal policy minimizes the sum of the variance of the discrete horizon $t$ estimators.
\end{definition}

\vspace{-0.2cm}
In the following, we adopt the trace criterion $\mathrm{tr}(\Sigma_{\mathrm{eff}}(\pi))$, corresponding to \emph{A-optimality}\footnote{In contrast to A-optimality, alternatives such as D-optimality (i.e., $\det\Sigma_{\mathrm{eff}}$, which corresponds to confidence-ellipsoid volume) or E-optimality (i.e., $\lambda_{\max}(\Sigma_{\mathrm{eff}})$ account for inter-horizon covariance structure but do not admit a closed-form solution; discussion in Appendix~\ref{app:add_optimal}.} in the classical theory of optimal experimental design~\citep{atkinson.2023}, which calculates the sum of component-wise asymptotic variances across horizons. Since the trace separates additively over $t$, it reduces to independently minimizing the marginal variance at each time horizon. The following result explicitly states the A-optimal treatment allocation policy for our task.
\begin{proposition}[A-optimal treatment allocation policy]
\label{prop:opt_assignment}
Assume $\alpha \in (0, 1/2)$ and $V_0(X)+V_1(X)>0$ for all $x\in\mathcal X$. Then, the allocation policy minimizing $\mathrm{tr}(\Sigma_{\mathrm{eff}}(\pi))$ subject to $\pi \in \Pi_\alpha$ is given by

\vspace{-0.6cm}
{\small\begin{align}\label{eq:optimal_pi}
\pi^\star(X) = \mathrm{clip}_\alpha\!\left(\frac{\sqrt{V_1(X)}}{\sqrt{V_1(X)}+\sqrt{V_0(X)}}\right), \quad \text{where } V_a(X) = \sum_{t=0}^{t_{\max}} v_{t,a}(X) = \mathrm{tr}(\Sigma_{a}(X)),
\end{align}}

\vspace{-0.2cm}
where $v_{t,a}(X)$ are the $t$-th diagonal entries of $\Sigma_a(X)$.
\end{proposition}

\vspace{-0.5cm}
\begin{proof}
See Appendix~\ref{app:proof_opt_assignment}.
\end{proof}
\vspace{-0.4cm}

\textbf{Interpretation:} Proposition~\ref{prop:opt_assignment} implies that the optimal policy $\pi^\star(X)$ allocates more probability mass to the arm $a$ with higher $V_a(X)$, where

\vspace{-0.6cm}
{\small\begin{align}
V_a(x) =& \sum_{t=0}^{t_{\max}} v_{t,a}(X) = \sum_{t=0}^{t_{\max}} \mathrm{Var}\!\left[S_t(X,a)\,\xi(\mathcal{O},\eta_t) \mid X, A=a\right]
= \sum_{t=0}^{t_{\max}} {S}_{t}(x,a)^2 \sum_{i=0}^{t} \frac{{\lambda}^S_{i}(x,a)}{{S}_{i}(x,a)\,{G}_{i-1}(x,a)}.
\end{align}}

\vspace{-0.4cm}
Unlike standard adaptive ATE estimation, $V_a(X)$ jointly captures two sources of uncertainty. (i)~\emph{survival uncertainty}: arms where the variance of the conditional survival probability (i.e., $\mathrm{Var}[S_t(X,a) \mid X, A=a]$) is large contribute more to $V_a(X)$; and (ii)~\emph{censoring uncertainty}: since $\xi(\mathcal{O},\eta_t)$ involves the inverse censoring weight $1/G_{i-1}(X,a)$, arms with larger censoring inflate $V_a(X)$, thus causing $\pi^\star$ to compensate by allocating more patients to that arm.

\begin{tcolorbox}[
    colback=NavyBlue!5!white,
    colframe=NavyBlue!75!black,
    title=\textbf{Illustrative example},
    arc=1mm,
    left=4.5pt, right=4.5pt, top=3pt, bottom=3pt
]

\begin{wrapfigure}[12]{l}{0.4\textwidth}
    \centering
    \vspace{-0.5cm}
    \includegraphics[width=\linewidth]{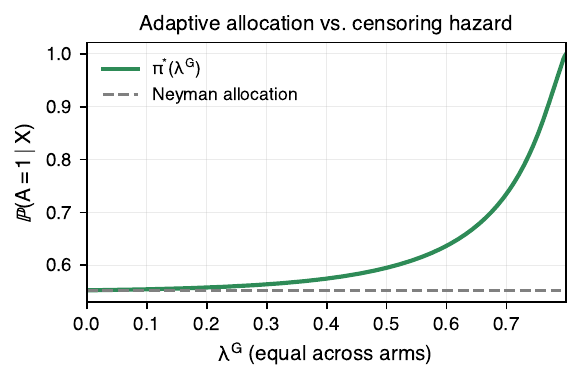}
    \vspace{-0.4cm}
    \caption{Censoring aware A-Optimal policy $\pi^\star(X)$ vs. censoring hazard.}
    \label{fig:illu_example}
    \vspace{-0.5cm}
\end{wrapfigure}

Consider the case where both arms share the same censoring hazard (i.e., $\lambda^G_t(x,0)=\lambda^G_t(x,1)=:\lambda^G$); yet where the event dynamics differ (i.e., $\lambda^S_t(x,0)\neq\lambda^S_t(x,1)$). One might expect that equal censoring renders $\pi^\star$ insensitive to $\lambda^G$, and that one would recover the classical Neyman allocation. However, Figure~\ref{fig:illu_example} refutes this: at $\lambda^G=0$, $\pi^\star$ coincides with the Neyman allocation (dashed), but diverges monotonically as $\lambda^G$ grows. The key mechanism is the following: equal $\lambda^G$ does not imply equal $G_{t-1}$ across arms, since the conditional censoring hazard $G_{t-1}(x,a)=\prod_{i = 0}^{t-1}1-\lambda^G_i/(1-\lambda^S_i(x, a))$ depends on arm-specific event dynamics, so the arm with higher event hazard experiences a larger increase in the censoring variance even under equal $\lambda^G$. We provide an additional example for intuition in Appendix~\ref{app:add_illu_eg}.

\vspace{0.1cm}
$\Rightarrow$ \emph{This demonstrates that censoring affects the optimal allocation through both $\lambda^G_t(X, a)$ and $\lambda^S_t(X, a)$; and that the standard Neyman allocation is suboptimal whenever $\lambda^G_t(X, a) > 0$}.
\end{tcolorbox}
 
\textbf{Neyman allocation as a special case:}
In the special case of no censoring in the trial (i.e., $G_{t-1}(x, a)\equiv 1$, $\widetilde T=T$ with $\Delta \equiv 1$), one can directly compute $V_a(X) = \sum_{t = 0}^{\tmax} \Var[S_t(X, a)\mid X, A = a] = \sum_{t = 0}^{\tmax} \Var \bigl(\mathbf 1\{T>t\} \big| X, A=a\bigr)$, so that the trace-optimal allocation in Proposition~\ref{prop:opt_assignment} reduces to
\vspace{-0.1cm}
{\small\begin{align}
\pi^\star(x)
\;=\;
\mathrm{clip}_\alpha\!\left(
\frac{\sqrt{\mathrm{tr}(\Var(\mathbf 1\{T>t\}\mid X=x, A=1))}}
{\sqrt{\mathrm{tr}(\Var(\mathbf 1\{T>t\}\mid X=x, A=1))}+\sqrt{\mathrm{tr}(\Var(\mathbf 1\{T>t\}\mid X=x, A=0))}}
\right) ,
\end{align}}

\vspace{-0.3cm}
which exactly matches the classical Neyman allocation for minimizing the variance of a difference-in-means estimator with a binary outcome $\mathbf 1\{T>t\}$~\citep{neyman.1934}. However, unlike standard Neyman allocation, which targets outcome variance alone, our policy accounts for a second source of uncertainty: $v_{t,a}(X)$ inflates whenever censoring is large in arm $a$, which directly forces the policy to compensate by allocating more units there.

\vspace{-0.2cm}
\section{The Adaptive Survival Estimator (ASE)}
\label{sec:ASE_full}
\vspace{-0.2cm}

We now propose the \textbf{Adaptive Survival Estimator (ASE)}, a sequential framework for learning the optimal allocation policy and estimating the average survival effect curve. Our goal is to minimize the semiparametric efficiency bound from Theorem~\ref{thm:eff_bound}. We proceed as follows: (1) We introduce the \emph{adaptive treatment allocation policy} in Section~\ref{sec:adaptive_assignment} and the adaptive version of \emph{average survival effect curve estimation} in Section~\ref{sec:adaptive_estimator}. We provide an extension to batch settings in Appendix~\ref{app:batch}.

\vspace{-0.2cm}
\subsection{Adaptive treatment allocation policy}
\label{sec:adaptive_assignment}
\vspace{-0.2cm}

Our framework operates in two steps. In the first $R_0$ rounds, units are assigned under a fixed policy $\pi_{\mathrm{init}}(X)$ (e.g., Bernoulli($1/2$) randomization) to accumulate sufficient history for stable estimation of the survival and censoring hazards $\lambda^S_t, \lambda^G_t$. From round $R_0+1$ onward, the policy switches to the data-driven policy $\pi_r(X\mid \mathcal{H}_{r-1})$ below, which sequentially approximates the A-optimal allocation of Proposition~\ref{prop:opt_assignment}. Concretely, we obtain a direct estimate $\widetilde{\pi}_r(X\mid \mathcal{H}_{r-1})$ by
\begin{align}
\label{eq:wide_hat_pi}
    \widetilde{\pi}_r(X \mid \gH_{r-1}) := \frac{\sqrt{\hat{V}_{1,r-1}(X)}}{\sqrt{\hat{V}_{1,r-1}(X)}+\sqrt{\hat{V}_{0, r-1}(X)}},
\end{align}

\vspace{-0.4cm}
where $V_{a, r-1}(X)$ is the trace of the conditional covariance matrix based on data in $\gH_{r-1}$. We then apply a truncation step to $\widetilde{\pi}_r(X\mid \gH_{r-1})$ to obtain the assignment policy $\pi_r(X\mid \gH_{r-1})$ at round $r$.

\textbf{Closed-form $V_{a, r-1}(X)$ estimation:} We now present the closed-form estimation of $V_{a, r-1}(X)$ by combining Theorem~\ref{thm:eff_bound} with A-optimal policy:
\vspace{-0.2cm}
\begin{align}
\label{eq:adap_v_a}
\hat{V}_{a, r-1}(X)
=\sum_{t=0}^{t_{\max}} \hat{S}_{t,r-1}(x,a)^2 \sum_{i=0}^{t} \frac{\hat{\lambda}^S_{i,r-1}(x,a)}{\hat{S}_{i,r-1}(x,a)\,\hat{G}_{i-1,r-1}(x,a)},
\end{align}

\vspace{-0.2cm}
where $\hat{S}_{i,r-1}(x,a)$ and $\hat{G}_{i-1,r-1}(x,a)$ are the survival- and censoring-hazard nuisances already required by the EIF in Eq.~\ref{eq:eif}, but now estimated sequentially from the accumulated history $\mathcal{H}_{r-1}$. Specifically, $\widetilde{\pi}_r(x)$ is a deterministic function of $\eta_{t, r-1}(X)$ alone. 
\vspace{0.1cm}
\begin{remark} 
\emph{These conclusions highlight the difference to existing work~\citep{oprescu.2025}, which typically requires a second-stage regression to estimate the variance target governing the optimal allocation. In contrast, Eq.~\ref{eq:adap_v_a} expresses $\hat{V}_{a,r-1}(x)$ as a closed-form function of the hazard nuisances $\hat\eta_{\,r-1}$, which are already used inside the EIF construction at round $r-1$. This yields two practical advantages: (i) no auxiliary second-stage learner for the variance is needed; (ii) the non-negativity of $\hat V_{a, r-1}$ holds by construction.}
\end{remark}

\textbf{Truncation for treatment allocation policy:} Extreme censoring or small within-stratum sample sizes can drive $\hat{V}_{a,r-1}(X)$ toward zero, thus causing the direct estimates of policy $\widetilde{\pi}_r(X\mid \mathcal{H}_{r-1})$ to collapse to the boundary of $[0,1]$. Following the literature~\citep{cook.2024, dai.2023, oprescu.2025}, we clip $\widetilde{\pi}_r$ to a shrinking interval, defining the final assignment policy as
\vspace{-0.2cm}
\begin{align}
\label{eq:trunc}
\pi_r(X\mid \gH_{r-1}):= \min \bigg\{1-\frac{1}{k_r}, \max\big\{\frac{1}{k_r}, \widetilde{\pi}_r(X\mid \gH_{r-1})\big\}\bigg\},
\end{align}

\vspace{-0.3cm}
where $k_r\in [2, \infty)$ is a non-decreasing sequence with $k_r \rightarrow \infty$ as $r \rightarrow \infty$. This keeps all assignment probabilities strictly interior, controls the growth of the censoring-weighted scores in the EIF, and is necessary for the theoretical guarantees in Section~\ref{sec:theory}.

\subsection{Sequentially cross-fitted adaptive ATE estimation}
\label{sec:adaptive_estimator}

\textbf{Why a custom cross-fitting approach is needed:} The reason is the sequential dependence: at round $r$, the estimated nuisances $\hat\eta_{t,r-1}$ are functions of the same accumulated history $\mathcal{H}_{r-1}$. The standard central limit theorems (CLTs) therefore do \underline{not} apply. Our approach resolves this via a near-martingale decomposition: we isolate an oracle term that satisfies a martingale CLT and a nuisance remainder whose bias vanishes at rate $o_p(R^{-1/2})$, which is guaranteed once each hazard estimator achieves the $o_p(r^{-1/4})$ rate imposed in Theorem~\ref{thm:adaptive_asymptotic_normality_cf}.

\textbf{Estimation via sequential cross-fitting:} To mitigate finite-sample bias and ensure the nuisance remainder is of order $o_p(R^{-1/2})$, we apply the sequential cross-fitting scheme of \citet{waudby-smith.2024} with $K=2$ folds. For each unit $r \in \mathbb{N}$, define the fold index $J_r := r \bmod 2 \in \{0,1\}$, and split $\gH_{r-1}$ into two temporal folds, i.e., $\gH_{r-1} = \gH_{r-1}^{(0)}\cup\gH_{r-1}^{(1)}$,

\vspace{-0.6cm}
\begin{align}
\gH_{r-1}^{(j)} := \bigl\{(X_{\tilde{r}}, A_{\tilde{r}}, \widetilde T_{\tilde{r}}, \Delta_{\tilde{r}}) : {\tilde{r}} \in [1,2,..., r-1],\ J_{\tilde{r}} = j\bigr\}, \qquad j \in \{0, 1\}.
\end{align}

\vspace{-0.4cm}
Given the cross-fitted nuisances $\hat\eta_{t,r-1}^{(-J_r)}$ and the adaptive policy $\pi_r(X \mid \gH_{r-1})$, we define the cross-fitted EIF pseudo-outcome at horizon $t$ by

\vspace{-0.6cm}
\begin{equation}
\hat\phi_{t,r}^{\mathrm{cf}}
:= 
\hat S_{t,r-1}^{(-J_r)}(X_r,1) - \hat S_{t,r-1}^{(-J_r)}(X_r,0) 
- \frac{\{A_r-\pi_r(X_r)\} \hat\xi\bigl(\gH_r,\hat\eta_{t,r-1}^{(-J_r)}\bigr)}{\pi_r(X_r)\{1-\pi_r(X_r)\}} \hat S_{t,r-1}^{(-J_r)}(X_r,A_r),
\label{eq:adaptive_cf_eif}
\end{equation}

\vspace{-0.2cm}
where $\hat\xi(\gH_r, \hat\eta_{t,r-1}^{(-J_r)})$ denotes the sequentially estimated version of $\xi(\gH_r,\eta_t)$ in Eq.~\ref{eq:eif}, with all nuisance components replaced by $\hat\eta_{t,r-1}^{(-J_r)}$. We now yield our ASE: our aim is to sequentially estimate the ATE from the non-centered efficient influence function in Eq.~\ref{eq:eif} with sequentially updated nuisance functions. In ASE, the estimator at horizon $t$ is defined as

\vspace{-0.5cm}
\begin{equation}
\hat{\tau}^{\mathrm{ASE}}_{t,R} = \frac{1}{R} \sum_{r=1}^R \hat{\phi}^{\mathrm{cf}}(\gH_{r}; \pi_r, \hat{\eta}_{t, r}^{\mathrm{cf}}), 
\end{equation}

\vspace{-0.3cm}
where $\hat{\eta}_{t,r-1}^{\mathrm{cf}} = \{\hat{\lambda}^{S, \mathrm{cf}}_{i, r-1}(X, 0), \hat{\lambda}^{S, \mathrm{cf}}_{i, r-1}(X, 1), \hat{\lambda}^{G, \mathrm{cf}}_{i, r-1}(X, 0), \hat{\lambda}^{G, \mathrm{cf}}_{i, r-1}(X, 1)\}_{i=0}^{t}$ denotes the estimates of the nuisance functions at round $r$ with sequential cross-fitting, which are constructed solely from the past data $\gH_{r-1}$. The treatment assignment policy $\pi_r(X\mid \gH_{r-1})$, defined by the experimenter based on the estimated optimal policy from Section~\ref{sec:adaptive_assignment}, is treated as known and does not require further estimation from data. The full procedure is summarized in Algorithm~\ref{alg:adaptive_survival}. 

\section{Theoretical guarantees}
\label{sec:theory}

This section provides several theoretical guarantees for the ASE. We establish asymptotic normality and an A-optimal semiparametric efficiency bound in Section~\ref{sec:theory_normality}, and robustness of the ASE under partial nuisance misspecification in Section~\ref{sec:theory_ortho}. Both results extend naturally to batch settings where policy updates occur at the end of each batch rather than every round, since the near-martingale decomposition depends only on $L_2$-consistency of the nuisance estimates and \emph{not} on the update frequency. In Appendix~\ref{app:asc}, we additionally derive asymptotically-valid time-uniform confidence sequences for the ASE, thereby enabling valid inference at arbitrary data-dependent stopping times.

\subsection{Asymptotic normality of the ASE}
\label{sec:theory_normality}

To derive our asymptotic guarantees\footnote{Throughout, estimators use sequential cross-fitting of Section~\ref{sec:adaptive_estimator}; we suppress the superscript ${}^{\mathrm{cf}}$ for brevity.}, we impose the following uniform overlap assumption on the survival and censoring hazards. This condition is common in semiparametric survival analysis~\citep{cui.2023, vanderlaan.2003} and ensures that the inverse probability of censoring weights in the Eq.~\ref{eq:eif} remain bounded.

\begin{assumption}[Uniform overlap up to horizon $t$]
\label{ass:uniform_overlap}
For each fixed $t \in \{0,\dots,t_{\max}\}$, there exist constants
$\underline{c}^S_t$, $\underline{c}^G_t>0$ such that, for all $a\in\{0,1\}$, all $x$, and all $i\le t$, we have $1-\lambda^S_i(x,a)\ge \underline{c}^S_t$ and $1-\lambda^G_i(x,a) / (1-\lambda^S_i(x,a))\ge \underline{c}^G_t$.
\end{assumption}

\begin{remark}[Necessity of stepwise overlap]
\emph{Assumption~\ref{ass:uniform_overlap} imposes bounds at every intermediate step $i \leq t$ rather than only at the endpoint, since $S_{t}$ and $G_{t-1}$ are cumulative products: a single step with $\lambda^S_i \to 1$ drives the $\xi(\gH_r;\eta)$ in the EIF to become unbounded. Hence, this constitutes the discrete-time survival analogue of the positivity condition standard in semiparametric survival analysis~\citep{cui.2023, vanderlaan.2003, westling.2024}. See Appendix~\ref{app:remark_overlap} for details.}
\end{remark}

\begin{theorem}[Asymptotic normality of the adaptive survival estimator] 
\label{thm:adaptive_asymptotic_normality_cf}
Suppose Assumptions~\ref{ass:causal}, \ref{ass:survival}, and Assumption~\ref{ass:uniform_overlap} hold, and let $k_r \in [2, \infty)$ be a non-decreasing truncation sequence satisfying $k_r \to \infty$ and $k_r = o(r^{1/4})$ as $r \to \infty$.  Suppose there exists a non-adaptive policy $\pi(X)\in [\varepsilon, 1-\varepsilon]$ for some $\varepsilon>0$ such that the adaptive assignment policy $\pi_r(X\mid \gH_{r-1})$ are $L_2$-consistent relative to truncation schedule, i.e., $k_r\|\pi_r - \pi\|_2 = o_p(1)$. Furthermore, assume that, for each fixed $t$ and treatment $a$, we have $\max_{0\leq i \leq \tmax} \|\hat{\lambda}^S_{i, r-1}(a, \cdot)-\lambda^S_i(a, \cdot)\|_2 =o_p(r^{-1/4})$ and $\max_{0\leq i \leq \tmax-1} \|\hat{\lambda}^G_{i, r-1}(a, \cdot)-\lambda^G_i(a, \cdot)\|_2 =o_p(r^{-1/4})$. Then, the adaptive estimator is asymptotically normal: 

\vspace{-0.8cm}
\begin{align}
\sqrt{R}\left(\hat{\tau}^{\mathrm{ASE}}_{t, R} - \tau_t \right) \rightarrow^d \mathcal{N}(0, V_{\mathrm{eff},t} (\pi)),
\end{align}

\vspace{-0.3cm}
where $V_{\mathrm{eff},t}(\pi) = \E \left[\frac{v_{t,1}(X)}{\pi(X)} + \frac{v_{t,0}(X)}{1-\pi(X)}\right] + b_t$ denotes the semiparametric efficiency bound for $\tau_t$ under policy $\pi$ and $b_t = \E[(S_t(X, 1) - S_t(X, 0) - \tau_t)^2]$. In particular, if we have $\pi=\pi^\star$, then $\hat{\tau}^{\mathrm{ASE}}_{t,R}$ attains the A-optimal semiparametric efficiency bound.
\end{theorem}

\vspace{-0.3cm}
\begin{proof}
See Appendix~\ref{app:proof_thm_asymptotic}.
\end{proof}
\vspace{-0.3cm}

The proof decomposes $\sqrt{R}(\hat{\tau}_t^{\mathrm{ASE}} - \tau_t) = \frac{1}{\sqrt{R}}\sum_{r=1}^R z_{t,r} + \frac{1}{\sqrt{R}}\sum_{r=1}^R m_{t,r}$ into an oracle term and a nuisance remainder. The oracle term $z_{t,r} = \phi_{t,r} - \tau_t$ forms a martingale difference sequence with respect to $\mathcal{H}_{r-1}$, since the adaptive policy $\pi_r$ is determined by $\mathcal{H}_{r-1}$ and the inverse censoring augmentation satisfies $\mathbb{E}[\xi(\gH_{r}, \eta_t) \mid A_r, X_r, \mathcal{H}_{r-1}] = 0$. The martingale central limit theorem (CLT) then yields asymptotic normality. The remainder $m_{t,r} = \hat{\phi}_{t,r} - \phi_{t,r}$ is controlled via the cause-specific structure of $G_{t-1}(X,a)$: the bias term vanishes at rate $o_p(r^{-1/2})$ under $o_p(r^{-1/4})$ hazard consistency, and the martingale fluctuation term is negligible by $L_2$-boundedness of the cross-fitted scores, thus requiring no Donsker conditions.

Importantly, the truncation does \emph{not} interfere with the optimal allocation in the limit. The reason is that, since $\pi^\star(X) \in [\alpha, 1-\alpha]$ and $1/k_r \to 0$, the clipping interval eventually contains $\pi^\star(X)$.

\begin{proposition}[Policy convergence under hazard consistency]
\label{prop:policy_convergence}
Suppose Assumptions~\ref{ass:causal}, \ref{ass:survival}, and \ref{ass:uniform_overlap} hold, and that the hazard estimators satisfy the rate conditions of Theorem~\ref{thm:adaptive_asymptotic_normality_cf}. Assume further that $\pi^\star(X) \in [\alpha, 1-\alpha]$ almost surely for some $\alpha > 0$. Then, the direct variance estimates satisfy $\|\hat{V}_{a,r-1} - V_a\|_2 = o_p(r^{-1/4})$ for each $a \in \{0,1\}$. Consequently, for any non-decreasing truncation sequence $k_r = o(r^{1/4})$, we have $k_r\|\pi_r - \pi^\star\|_2 = o_p(1)$.
\end{proposition}
\vspace{-0.3cm}
\begin{proof}
See Appendix~\ref{app:proof_policy_convergence}.
\end{proof}
\vspace{-0.2cm}

\vspace{-0.3cm}
\subsection{Consistency guarantees under partial nuisance misspecification}
\label{sec:theory_ortho}
\vspace{-0.2cm}

As shown in Theorem~\ref{thm:adaptive_asymptotic_normality_cf}, the convergence rate of ASE is primarily governed by the estimation error of $\hat{\lambda}^S_t(X, \cdot)$ and $\hat{\lambda}^G_t(X, \cdot)$. This reflects a \emph{robustness} property formalized in the following corollary: misspecification of the censoring hazard $\hat{\lambda}^G_t(X, \cdot)$ alone does \emph{not} affect consistency, provided the event hazard $\hat{\lambda}^S_t(X, \cdot)$ is estimated consistently.

\begin{corollary}[Robustness to the censoring hazard]
\label{cor:orth_lambdaG}
Suppose that the event hazard estimators are consistent, i.e., $\|\hat\lambda^S_{t,r-1}-\lambda^S_t\|_2=o_p(1)$ for all $t$, and that the censoring hazard estimators satisfy $\|\hat\lambda^G_{t,r-1}-\lambda^G_t\|_2=o_p(1)$ for all $t$. Then, for each $t \leq t_{\max}$,
\vspace{-0.2cm}
{\footnotesize\begin{align} 
&|\hat\tau_{t, R}^{\mathrm{ASE}}-\tau_{t}| = O_p(R^{-1/2}) \\ 
&+ O_p\!\left(\Bigl(\sum_{i=0}^{t} \|\hat{\lambda}^S_{i,r-1}(a,\cdot)-\lambda^S_i(a,\cdot)\|_2\Bigr)^2 + \sum_{i=0}^{t}\|\hat{\lambda}^S_{i,r-1}(a,\cdot)-\lambda^S_i(a,\cdot)\|_2 \sum_{i=0}^{t-1}\|\hat{\lambda}^G_{i,r-1}(a,\cdot)-\lambda^G_i(a,\cdot)\|_2\right). \notag \vspace{-0.4cm}
\end{align}} 

\vspace{-0.4cm}
In particular, there is no standalone first-order term $O_p(\|\hat\lambda^G_t - \lambda^G_t\|_2)$; any perturbation in $\lambda^G$ enters only through its interaction with the event nuisance, reflecting the structure of misspecification of censoring hazard.
\end{corollary}

\vspace{-0.4cm}
\begin{proof}
See Appendix~\ref{app:orthogonality}.
\end{proof}
\vspace{-0.4cm}

For the no-ties case, we can even provide a double-robustness result (see Appendix~\ref{app:no-ties}, where $\Delta = \mathbf{1}(T < C)$ excludes simultaneous event and censoring times. In this setting, the censoring survival function simplifies to $G_{t-1}(x,a) = \prod_{i=0}^{t-1}(1 - \lambda^G_i(x,a))$, which is independent of the event hazard $\lambda^S$. As a result, the ASE is doubly-robust (i.e., to misspecification of both $\lambda^S$ and $\lambda^G$), and the remainder vanishes whenever either nuisance is consistently estimated, regardless of the other.

\section{Experimental results}
\label{sec:exp_results}
\vspace{-0.2cm}

\begin{figure}[H]
    \centering
    \includegraphics[width=1\linewidth]{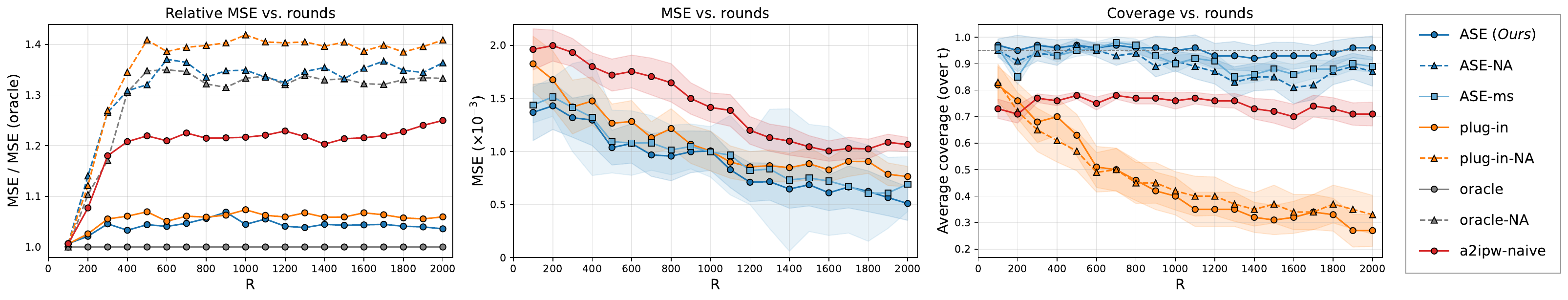}
    \vspace{-0.5cm}
    \caption{\textbf{Synthetic experiments:} 
    \protect\circledblue{a}~\emph{(left)}: Relative MSE with respect to Oracle; ASE achieves the lowest error among baselines. 
    \protect\circledblue{b}~\emph{(middle)}: MSE across rounds; ASE converges fastest while ASE-MS remains consistent under nuisance misspecification.
    \protect\circledblue{c}~\emph{(right)}: Empirical coverage of nominal 95\% CIs; ASE achieves nominal coverage while plug-in and A2IPW-NA\"{i}ve deteriorate as $R$ grows.}
    \label{fig:syn_exp_results}
    \vspace{-0.4cm}
\end{figure}

\vspace{-0.2cm}
\textbf{Baselines:} We compare our estimator (\textbf{ASE}) against several baselines: (1)~its non-adaptive counterpart (\textbf{ASE-NA}), which assigns the treatment uniformly at random;~(2)~the plug-in estimator from Eq.~\ref{eq:plugin}~(\textbf{plug-in}); and (3)~its non-adaptive version (\textbf{plug-in-NA}); (4)~the \textbf{A2IPW-NA\"{i}ve} estimator from~\citet{kato.2024} as the Neyman-allocation which ignores censoring in allocation policy; we also report two oracle baselines for comparison: (5) a fully use oracle-efficient estimator that uses the oracle nuisance function (\textbf{Oracle}) and (6) a non-adaptive version (\textbf{Oracle-NA}). To assess the benefit of robustness, we also evaluate a misspecified version of ASE, denoted as \textbf{ASE-MS}, in which the $\lambda^G_t(X, a)$ estimator is deliberately misspecified. 

\textbf{Implementation details:} We estimate the nuisance functions using LightGBM, trained with standard gradient-boosted tree objectives for classification (see Appendix~\ref{app:implementation_details} for details).

\textbf{Synthetic data:} We follow the common practice in adaptive experimentation~\citep{oprescu.2025, zhang.2025} and evaluate the effectiveness of our framework on both synthetic and semi-synthetic datasets to benchmark against ground-truth results. We provide the estimated average survival effect curves for synthetic and semi-synthetic data in Appendix~\ref{app:survival_curves}. In our experiments, we aim to evaluate three research questions: 

\circledblue{a}~\emph{How much do we benefit from adaptive allocation?}~$\Rightarrow$~\textbf{Effectiveness of A-optimal treatment allocation}: For this, we report the relative MSE with respect to the Oracle estimator in Fig.~\ref{fig:syn_exp_results}~(a). Our ASE achieves a relative MSE of approximately $1.05\times$ of the oracle at $R=2000$, which nearly matches oracle performance despite using estimated nuisances. In contrast, the non-adaptive baselines (ASE-NA and plug-in-NA) stabilize at $1.35\times$ and $1.40\times$ of the oracle MSE, respectively, maintaining a persistent efficiency gap. The a2ipw-NA\"{i}ve method falls in an intermediate position at roughly $1.20\times$--$1.25\times$, which confirms that our censoring-aware A-optimal allocation (and not just adaptive assignment alone) is the key driver of the efficiency gains over uniform randomization.

\circledblue{b}~\emph{Does ASE improve estimation accuracy over rounds?} $\Rightarrow$ \textbf{Estimation accuracy}: We thus plot the MSE against $R$ in Fig.~\ref{fig:syn_exp_results}~(b). First, ASE, ASE-MS, plug-in, and a2ipw-NA\"{i}ve converge to the true $\tau$. Out of them, ASE achieves the lowest error and thus performs best. The results further highlight the importance of our robustness property: ASE-MS remains consistent and even achieves a similar speed of convergence as ASE, despite one nuisance component being misspecified.

\circledblue{c}~\emph{Is ASE inferentially valid?} $\Rightarrow$ \textbf{Inferential validity}: We here report empirical coverage of nominal 95\% confidence intervals (CIs) in Fig.~\ref{fig:syn_exp_results}~(c). Consistent with our theoretical guarantee, our ASE and ASE-NA achieve nominal coverage. In contrast, the other methods fail to do so. For example, the plug-in, plug-in-NA, and a2ipw-NA\"{i}ve perform poorly as $R$ grows, which is due to finite-sample bias and censoring bias.

\textbf{Semi-synthetic data:} We also evaluate ASE on a semi-synthetic dataset based on the Twins dataset as in \cite{frauen.2025, louizos.2017} to demonstrate the applicability to medical data. The results are in Appendix~\ref{app:twins_result}.

\textbf{Results:} Our findings are consistent with the synthetic setting: (i) adaptive treatment allocation policy improves estimation efficiency, (ii) ASE achieves superior coverage and consistency, and (iii) robustness holds under partial nuisance misspecification.

\textbf{Conclusion:} Our work highlights the potential broader impact of adaptive experimentation to improve clinical trials with survival data, especially in settings where efficient evidence generation is critical. \textbf{Limitations:} Future work could be extended to multi-arm trials and demonstrate its practical value in real clinical settings with limited and heterogeneous patient populations.

\newpage
\section*{Acknowledgement}
Our research is supported by the DAAD programme Konrad Zuse Schools of Excellence in Artificial Intelligence, sponsored by the Federal Ministry of Research, Technology and Space.

\bibliographystyle{abbrvnat}
\bibliography{literature}


\newpage
\appendix
\section{Notation}
\label{app:notation}
{
\centering
\captionof{table}{Notation}
\resizebox{\textwidth}{!}{%
\renewcommand{\arraystretch}{1.5}
\arrayrulecolor{gray!30}
\begin{tabular}{ll}
\arrayrulecolor{black}\toprule\arrayrulecolor{gray!30}
\textbf{Symbol} & \textbf{Description} \\
\arrayrulecolor{black}\toprule\arrayrulecolor{gray!30}
$X$ & Observed covariates, $X \in \mathcal{X} \subseteq \mathbb{R}^p$ \\ \hline
$A$ & Binary treatment, $A \in \{0, 1\}$ \\ \hline
$T$ & Event time of interest (e.g., overall survival time (OS), disease-free survival time (DFS)) \\ \hline
$C$ & Censoring time (e.g., dropout) \\ \hline
$\widetilde{T}$ & Observed time, $\widetilde{T} = \min\{T, C\}$ \\ \hline
$\Delta$ & Event indicator, $\Delta = \mathbf{1}(T \leq C)$ \\ \hline
$\mathcal{O}$ 
& Observed population, $\mathcal{O} = (X, A, \widetilde{T}, \Delta)$ \\ \hline
$\gH_r$ 
& Observed accumulated data at round $r$, $\gH_r = \left\{(x_1,a_1,\tilde{t}_1,\delta_1), \dots, (x_{r},a_{r}, \tilde{t}_{r},\delta_{r}) \right\}$\\ \hline
$\vtau$, $\evtau_t$, $\tau_t(x)$
& \makecell[l]{Average survival effect curve vector $\mathbf{\tau} = (\tau_0, \dots, \tau_{t_{\max}})^\top$ \\and its $t$-th entry $\tau_t = \mathbb{E}[S_t(X,1) - S_t(X,0)]$ \\
Conditional ATE at horizon $t$: $\tau_t(x) = \mathbb{P}(T(1)>t \mid X=x) - \mathbb{P}(T(0)>t \mid X=x)$} \\ \hline
$\hat{\tau}^{\mathrm{ASE}}_{t,R}$ 
& ASE estimator at horizon $t$ after $R$ rounds \\ \hline
$\bm{S}(x,a)$, $S_t(x,a)$
& \makecell[l]{Survival function vector $\bm{S}(x,a) = (S_0(x,a),\dots,S_{t_{\max}}(x,a))^\top$ \\ and its $t$-th entry $S_t(x,a) = \mathbb{P}(T > t \mid X=x, A=a)$} \\ \hline
$\hat{\bm{S}}_r(x,a)$, $\hat{S}_{t,r}(x,a)$, $\hat{S}^{(-J_r)}_{t,r}(x,a)$
& \makecell[l]{Estimated survival function vector, its $t$-th entry, \\ and its cross-fitted version using fold $\mathcal{H}_{r-1}^{(-J_r)}$} \\ \hline
$\bm{G}(x,a)$, $G_t(x,a)$
& \makecell[l]{Censoring survival function vector $\bm{G}(x,a) = (G_0(x,a),\dots,G_{t_{\max}}(x,a))^\top$ \\ and its $t$-th entry $G_t(x,a) = \mathbb{P}(C > t \mid X=x, A=a)$} \\ \hline
$\hat{\bm{G}}_r(x,a)$, $\hat{G}_{t,r}(x,a)$, $\hat{G}^{(-J_r)}_{t,r}(x,a)$
& \makecell[l]{Estimated censoring survival function vector, its $t$-th entry, \\ and its cross-fitted version using fold $\mathcal{H}_{r-1}^{(-J_r)}$} \\ \hline
$\bm{\lambda}^S(x,a)$, $\lambda^S_t(x,a)$
& \makecell[l]{Event hazard vector $\bm{\lambda}^S(x,a) = (\lambda^S_0(x,a),\dots,\lambda^S_{t_{\max}}(x,a))^\top$ \\ and its $t$-th entry $\lambda^S_t(x,a) = \mathbb{P}(\widetilde{T}=t,\Delta=1\mid\widetilde{T}\geq t, X=x, A=a)$} \\ \hline
$\hat{\bm{\lambda}}^S_r(x,a)$, $\hat{\lambda}^S_{t,r}(x,a)$, $\hat{\lambda}^{S,(-J_r)}_{t,r}(x,a)$
& \makecell[l]{Estimated event hazard vector, its $t$-th entry, \\ and its cross-fitted version using fold $\mathcal{H}_{r-1}^{(-J_r)}$} \\ \hline
$\bm{\lambda}^G(x,a)$, $\lambda^G_t(x,a)$
& \makecell[l]{Censoring hazard vector $\bm{\lambda}^G(x,a) = (\lambda^G_0(x,a),\dots,\lambda^G_{t_{\max}}(x,a))^\top$ \\ and its $t$-th entry $\lambda^G_t(x,a) = \mathbb{P}(\widetilde{T}=t,\Delta=0\mid\widetilde{T}\geq t, X=x, A=a)$} \\ \hline
$\hat{\bm{\lambda}}^G_r(x,a)$, $\hat{\lambda}^G_{t,r}(x,a)$, $\hat{\lambda}^{G,(-J_r)}_{t,r}(x,a)$
& \makecell[l]{Estimated censoring hazard vector, its $t$-th entry, \\ and its cross-fitted version using fold $\mathcal{H}_{r-1}^{(-J_r)}$} \\ \hline
$\eta_t$, $\hat{\eta}_{t,r-1}$, $\hat\eta_{t,r-1}^{(-J_r)}$ 
& \makecell[l]{Oracle nuisance vector $\eta_t = \{\lambda^S_i(\cdot,\cdot), \lambda^G_i(\cdot,\cdot)\}_{i=0}^t$ \\  \hline
and its estimate from $\mathcal{H}_{r-1}$: $\hat{\eta}_{t,r-1} = \{\hat{\lambda}^S_{i,r-1}(\cdot,\cdot), \hat{\lambda}^G_{i,r-1}(\cdot,\cdot)\}_{i=0}^t$\\
and its cross-fitted version trained on fold $\mathcal{H}_{r-1}^{(J_r)}$} \\ \hline
$\bm{\xi}(\mathcal{O}, \eta_{\tmax})$ 
& IPCW augmentation vector: $\bm{\xi}(\mathcal{O}, \eta_{\tmax}) = (\xi(\mathcal{O}, \eta_0), \dots, \xi(\mathcal{O}, \eta_\tmax))^\top$\\ \hline
$\xi(\mathcal{O}, \eta_t)$, $\xi(\mathcal{H}_r, \eta_t)$ 
& \makecell[l]{IPCW augmentation term in the EIF from\\ observed population and accumulated data, respectively}\\ \hline
$\hat{\xi}(\mathcal{O}, \eta_t)$, $\hat{\xi}(\mathcal{H}_r, \eta_t)$, $\hat{\xi}(\mathcal{H}_r, \eta_t^{\mathrm{cf}})$ 
& \makecell[l]{Estimated IPCW augmentation term in the EIF from\\ observed population and accumulated data,\\ and with sequential cross-fitting, respectively.}\\ \hline
$\bm{\phi}(\cdot;\tau,\eta_{\tmax})$, $\phi_t(\cdot;\tau_t,\eta_t)$
& \makecell[l]{EIF vector $\bm{\phi} = (\phi_0,\dots,\phi_{t_{\max}})^\top$ and its $t$-th entry (non-centered EIF at horizon $t$, see Eq.~\ref{eq:eif})}\\ \hline
$\phi_{t, r}$, $\hat{\phi}_{t,r}$, $\hat{\phi}_{t,r}^{\mathrm{cf}}$  & \makecell[l]{Denotes non-centered EIF with oracle nuisances $\phi_{t, r} = \phi(\cdot;\tau_t, \pi_r, \eta_t)$\\
non-centered EIF with estimated nuisances $\hat{\phi}_{t, r} = \phi(\cdot;\tau_t, \pi_r, \hat{\eta}_{t, r})$\\
$\hat{\phi}_{t,r}^{\mathrm{cf}}$ denotes the cross-fitted nuisance at round $r$} \\ \hline
$\pi(x)$ 
& Treatment assignment probability: $\mathbb{P}(A=1 \mid X=x)$ \\ \hline
$\pi^\star(x)$ 
& A-optimal treatment allocation policy (Proposition~\ref{prop:opt_assignment}) \\ \hline
$\pi_r(x \mid \mathcal{H}_{r-1})$ 
& Adaptive policy at round $r$ \\ \hline
$\widetilde{\pi}_r(x \mid \mathcal{H}_{r-1})$ 
& Plug-in (pre-truncation) policy at round $r$ \\ \hline
$\Pi_\alpha$ 
& Feasible policy class: $\{\pi : \alpha \leq \pi(x) \leq 1-\alpha\}$ \\ \hline
$R$, $R_0$ 
& Total number of rounds and burn-in period (initial exploration) \\ \hline
$k_r$ 
& Truncation parameter at round $r$ (ensures $\pi_r \in [1/k_r, 1-1/k_r]$) \\ \hline
$J_r$ 
& Sequential cross-fitting fold index, $J_r := r \bmod 2$ \\ \hline
$\mathcal{H}_{r-1}^{(j)}$
& Temporal cross-fitting fold $j \in \{0,1\}$ \\ \hline
$\Sigma_{\mathrm{eff}}(\pi)$, $\Sigma_{a}(X)$ 
& \makecell[l]{Semiparametric efficiency bound matrix\\
$\Sigma_{\mathrm{eff}}(\pi) = \mathbb{E}\bigl[(\bm{\phi}(\mathcal{O};\tau,\eta_{t_{\max}}) - \tau)(\bm{\phi}(\mathcal{O};\tau,\eta_{t_{\max}}) - \tau)^\top\bigr] \in \mathbb{R}^{(t_{\max}+1)\times(t_{\max}+1)}$\\
$\Sigma_a(X) = \Var(\bm{S}(X,a) \odot \bm{\xi}(\gO;\bm{\eta})\mid X, A=a)\in \mathbb{R}^{(t_{\max}+1)\times(t_{\max}+1)}$
} \\ \hline
$v_{t,a}(x)$, $V_a(x)$ 
& \makecell[l]{$t$-th component of $\Sigma_a(X)$,\\
Variance target $V_a(x) = \sum_{t=0}^{t_{\max}} v_{t,a}(x) = \mathrm{tr}(\Sigma_a(x))$}\\ \hline
$V_{\mathrm{eff},t}(\pi)$ & and its $t$-th diagonal entry $\Sigma_{\mathrm{eff}}(\pi) = [\Sigma_{\mathrm{eff}}(\pi)]_{tt}$ \\ \hline
$\|\cdot\|_2$ 
& $L_2$ norm: $\|g\|_2 := \mathbb{E}[g(X)^2]^{1/2}$ \\ \hline
$o_p(1),\, O_p(\cdot)$ 
& Standard probabilistic asymptotic order notation \\ 
\arrayrulecolor{black}\bottomrule\arrayrulecolor{gray!30}
\end{tabular}
}
}

\newpage
\section{Extended related work}
\label{app:extended_related_work}

\begin{table}[htbp]
\centering
\begin{minipage}{\linewidth}
\resizebox{1\textwidth}{!}{%
\renewcommand{\arraystretch}{1.5}
\footnotesize
\begin{tabular}{llccc}
\arrayrulecolor{black}\toprule
\textbf{Literature stream}
& \textbf{Allocation objective}
& \textbf{Censored outcome}
& \textbf{Efficiency bound}
& \textbf{Model-agnostic} \\
\hline
RAR / CARA designs~\citep{hu.2006, rosenberger.2001, thompson.1933}
& welfare optimization & \xmark & \xmark & \xmark$^\dagger$ \\ 
CARA designs for survival~\citep{mukherjee.2025}
& welfare optimization & \cmark & \xmark & \xmark$^\dagger$\\
\hline
Adaptive ATE estimation (two-stage)~\citep{hahn.2011}
& estimation of outcome variability & \xmark & \cmark & \cmark \\
Adaptive ATE estimation (multi-round)~\citep{cook.2024, dai.2023, kato.2024}
& estimation of outcome variability & \xmark & \cmark & \cmark \\
\midrule
\textbf{Ours}
& estimation of survival \& censoring variability & \cmark & \cmark & \cmark\\
\bottomrule
\end{tabular}
}
\parbox{\linewidth}{\raggedright{\tiny{\textsuperscript{$\dagger$} RAR/CRAR rely on parametric response models; \citet{mukherjee.2025} is Cox proportional hazards (PH) model specific.}}}
\vspace{0.01cm}
\caption{\textbf{Overview of adaptive randomization literature streams.}}
\label{tab:related_work}
\vspace{-0.6cm}
\end{minipage}
\end{table}

\textbf{Estimation of treatment effects with censored outcomes.} A related line of work extends causal inference to survival settings, focusing on estimating treatment effects such as the CATE under censoring~\citep{hu.2021}. Methods here divide broadly into (i)~model-based learners, including tree-based~\citep{cui.2023, henderson.2020, tabib.2020, zhang.2017} and neural-network-based approaches~\citep{curth.2021, katzman.2018, schrod.2022}, and (ii)~(orthogonal) meta-learners specifically designed for censored time-to-event data~\citep{frauen.2025, gao.2022, vanderlaan.2003, xu.2024, xu.2022}. While these methods provide principled estimation under censoring, they operate under a \emph{fixed} treatment assignment mechanism and focus on point estimation, typically requiring censoring-independence assumptions that limit applicability. A complementary line develops inferential methods via conformal inference~\citep{candes.2023, davidov.2025, gui.2024}, thereby targeting prediction intervals for individual treatment effects rather than population-level causal estimands. $\Rightarrow$ Neither stream addresses adaptive experimental design, and neither leverages semiparametric efficiency bounds to guide treatment allocation under censoring.

\textbf{Bayesian response-adaptive randomization.} In parallel, Bayesian RAR (BRAR) and Thompson sampling~\citep{kaufmann.2017, lattimore.2020} have seen substantial uptake in both the machine learning and clinical trial communities~\citep{barker.2009, kim.2011}, driven by empirical gains in system performance and patient welfare. However, BRAR methods are primarily designed to maximize reward or accelerate identification of superior treatments, rather than to minimize the asymptotic variance of a causal estimand. As a result, they do not target the semiparametric efficiency bound, and their allocation rules do not account for the additional variance inflation induced by censoring. Our framework is complementary in spirit but distinct in objective: we derive allocation policy that are provably optimal with respect to the efficiency bound for the average survival effect curve.

\textbf{CRAR with survival outcomes:} The closest work to ours is~\citet{mukherjee.2025}, who propose CRAR designs for survival data under a proportional hazards assumption, using a sequentially estimated Cox model to drive allocation. However, their approach is driven by model-based estimates rather than optimizing allocation for statistical efficiency, and therefore does not target the semiparametric efficiency bound.

\textbf{Adaptive allocation for efficient ATE estimation.} There are several extensions that add principled policy truncation with anytime-valid inference~\citep{cook.2024}, and establish sublinear or logarithmic Neyman regret via clipped and optimistic algorithms~\citep{chen.2025,dai.2023, neopane.2024, neopane.2025, noarov.2025}. Further work covers covariate-adaptive and off-policy settings~\citep{lee.2024b, li.2024} and develops TMLE-based guarantees for sequential experiments~\citep{zhang.2025}.

\textbf{Confidence sequences and anytime-valid inference.} Unlike fixed-time confidence intervals, a confidence sequence (CS) provides a time-uniform coverage guarantee, thereby enabling valid inference at any data-dependent stopping time~\citep{howard.2020}. Martingale methods and empirical Bernstein inequalities have been central to recent constructions of CSs for influence-function-based estimators~\citep{cook.2024, waudby-smith.2024}, and \citet{lindon.2026} extends this to the delayed outcome regime. Our approach builds on these tools to construct an asymptotic CS for the ASE estimator, where the key challenge is simultaneously accommodating adaptive allocation, right censoring, and the almost-sure consistency requirements that time-uniform inference imposes on the nuisance remainder.

\newpage
\section{Additional illustrative example}
\label{app:add_illu_eg}

Consider an example with \emph{arm-dependent censoring}, the two treatment arms experience different rates of dropout, as when treatment improves retention ($G_t(X,1) > G_t(X,0)$) or induces side effects leading to dropout ($G_t(X,1) < G_t(X,0)$). Suppose the censoring survival functions have time-invariant ratio, $G_{t}(x,0)/G_{t}(x,1) = g(x)$ for all $t$, while event dynamics $S_t(x,a) = s_{t,a}$ may differ across arms. Proposition~\ref{prop:opt_assignment} yields the closed-form optimum $\pi^\star(x) = \mathrm{clip}_\alpha\!\left(\frac{\sqrt{K_1}}{\sqrt{K_0/g_t(x)}+\sqrt{K_1}}\right)$, where $K_a(x)$ captures arm-$a$ event dynamics and is censoring-independent up to a common factor that cancels in $\pi^\star$ (proof as following). Figure~\ref{fig:add_illu_ratio} plots $\pi^\star$ against difference value of $g(x) = G_t(X,0)/G_t(X,1)$, with closed form $\pi^\star(x) = \sqrt{g(x)}/(\sqrt{\kappa}+\sqrt{g(x)})$ where $\kappa := K_0/K_1$ captures the event-information ratio across arms. Under heavier censoring in one arm, allocating more units to that arm compensates for the censoring variance inflation, preserving estimator efficiency.

\begin{figure}[htbp]
    \centering
    \includegraphics[width=0.5\linewidth]{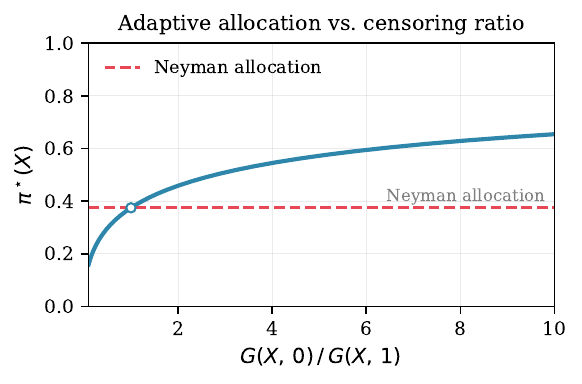}
    \vspace{-0.2cm}
    \caption{Optimal allocation $\pi^\star$ under arm-dependent censoring ratio $g(x)=G_t(x,0)/G_t(x,1)$: $\pi^\star$ departs from uniform allocation $1/2$ and Neyman-allocation as the censoring asymmetry increases.}
    \label{fig:add_illu_ratio}
\end{figure}

\textbf{Proof.} We derive the closed-form expression for $\pi^\star$ claimed in the motivating illustration above. Assume that the event dynamics may differ across arms and that the censoring survival functions have a time-invariant ratio: $S_t(x, a) = s_{t, a}$, $G_{i-1}(x,0)/ G_{i-1}(x,1) = g(x)$ for all $i$.
First, we write out the expression from Eq.~\ref{eq:optimal_pi} as
\begin{align}
v_{t,a}(x)
\;=\;
S_t(x,a)^2 \sum_{i=0}^{t}
\frac{\lambda_i^S(x,a)}{S_i(x,a)\,G_{i-1}(x,a)}
\;=\;
\frac{1}{g(x)^{\mathbf{1}\{a=0\}}} \cdot s_{t, a}^2 \sum_{i=0}^{t} \frac{\lambda_{i, a}}{s_{i, a}\, G_{i-1}(x,1)},
\end{align}
where we pulled the time-constant factor $g(x)$ out of the sum (for $a=0$) using the ratio assumption $G_{i-1}(x,0) = g(x)\, G_{i-1}(x,1)$. Then, we prove that
\begin{align}
s_{t,a}^2 \sum_{i=0}^{t} \frac{\lambda_{i, a}}{s_{i, a}}
\;=\;
s_{t,a}(1-s_{t,a}).
\label{eq:telescope_identity}
\end{align}
The key observation is that $\lambda_{i, a}/s_{i, a}$ is a first-order difference of $1/s_{i, a}$. Using $s_{i, a} = s_{i-1, a}(1-\lambda_{i, a})$, we have $s_{i-1, a} - s_{i, a} = s_{i-1, a}\lambda_{i, a}$, so
\begin{align}
\frac{1}{s_{i, a}} - \frac{1}{s_{i-1, a}}
\;=\;
\frac{s_{i-1, a}-s_{i, a}}{s_{i, a}\,s_{i-1, a}}
\;=\;
\frac{\lambda_{i, a}}{s_{i, a}}
\qquad (i\ge 1).
\end{align}
Then, we have
\begin{align}
\sum_{i=0}^{t} \frac{\lambda_{i, a}}{s_{i, a}}
\;=\;
\frac{1}{s_{t, a}} - \frac{1}{s_{-1, a}}
\;=\;
\frac{1-s_{t, a}}{s_{t, a}},
\end{align}
and multiplying by $s_{t, a}^2$ yields Eq.~\ref{eq:telescope_identity}. By absorbing the shared time-varying factor $G_{i-1}(x,1)$ into a common weighting that cancels in the final ratio, we write
\begin{align}
V_a(x)
\;=\;
\sum_{t=0}^{t_{\max}} v_{t,a}(x)
\;=:\;
\frac{K_a(x)}{g(x)^{\mathbf{1}\{a=0\}}},
\qquad
K_a(x) := \sum_{t=0}^{t_{\max}} s_{t,a}^2 \sum_{i=0}^{t} \frac{\lambda_{i,a}}{s_{i,a}\, G_{i-1}(x,1)},
\end{align}
so that $V_1(x) = K_1(x)$ and $V_0(x) = K_0(x)/g(x)$. Substituting into Proposition~\ref{prop:opt_assignment}, we arrive at
\begin{align}
\frac{\sqrt{V_1(x)}}{\sqrt{V_1(x)}+\sqrt{V_0(x)}}
\;=\;
\frac{\sqrt{K_1}}{\sqrt{K_1}+\sqrt{K_0/g(x)}}.
\end{align}
Applying the overlap constraint, we yield
\begin{align}
\pi^\star(x)
\;=\;
\mathrm{clip}_\alpha\!\left(
\frac{\sqrt{K_1(x)}}{\sqrt{K_0(x)/g(x)}+\sqrt{K_1(x)}}
\right).
\end{align}
\qed

\newpage
\section{Additional theoretical results for optimal data acquisition}
\label{app:add_optimal}
In the main text, we focus on \emph{A-optimal} allocation, i.e., minimizing $\mathrm{tr}(\Sigma_{\mathrm{eff}}(\pi))$, which targets small marginal variances across horizons and admits the closed-form solution in Proposition~\ref{prop:opt_assignment}. Alternative classical criteria from optimal experimental design are \emph{D-optimality} and \emph{E-optimality}, which target the joint uncertainty of the full survival effect curve vector $\tau = (\tau_0, \ldots, \tau_{t_{\max}})^\top$~\citep{atkinson.2023}, or control the worst-case linear combination of survival effect curve contrasts. We extend our derivations to D-optimality and E-optimality below. 

\subsection{D-optimality}

\begin{definition}[D-optimal allocation policy] \itshape
A feasible allocation policy $\pi^D$ is \emph{D-optimal} if \begin{align}\pi^D \in \arg\min_{\pi \in \Pi_\alpha} \det(\Sigma_{\mathrm{eff}}(\pi)) \quad \text{equivalently} \quad \pi^D \in \arg\min_{\pi \in \Pi_\alpha} \log \det(\Sigma_{\mathrm{eff}}(\pi)),\end{align} where $\Pi_\alpha = \{\pi : \alpha \leq \pi(x) \leq 1-\alpha \text{ for all } x\}$.
\end{definition}

\textbf{Interpretation:} The confidence ellipsoid for $\tau$ has volume proportional to $\det(\Sigma_{\mathrm{eff}}(\pi))^{1/2}$ (up to constants independent of $\pi$). Hence, D-optimal allocation tends to yield small \emph{joint} uncertainty across all time horizons simultaneously, which is desirable when the full survival effect curve is the object of interest rather than any individual $\tau_t$.

\begin{theorem}[D-optimal allocation policy]
\label{thm:d_optimal}
Assume $\Pi_\alpha \neq \emptyset$ and that Assumptions~\ref{ass:causal} and~\ref{ass:survival} hold. Then, any D-optimal policy $\pi^D$ satisfies, for a.e.\ $x$, 
\begin{align}
\pi^D(x) = \mathrm{clip}_\alpha\!\left(\frac{\sqrt{\mathrm{tr}\!\left(\Sigma_{\mathrm{eff}}(\pi^D)^{-1} \Sigma_1(x)\right)}}{\sqrt{\mathrm{tr}\!\left(\Sigma_{\mathrm{eff}}(\pi^D)^{-1} \Sigma_1(x)\right)} + \sqrt{\mathrm{tr}\!\left(\Sigma_{\mathrm{eff}}(\pi^D)^{-1} \Sigma_0(x)\right)}}\right).
\end{align} 
The characterization is a fixed-point equation since $\Sigma_{\mathrm{eff}}(\pi^D)$ depends on $\pi^D$.\end{theorem}

\begin{remark} The form of $\pi^D$ parallels the A-optimal rule in Proposition~\ref{prop:opt_assignment}, but with $V_a(X) = \mathrm{tr}(\Sigma_a(X))$ replaced by $\mathrm{tr}(\Sigma_{\mathrm{eff}}(\pi^D)^{-1} \Sigma_a(x))$. This difference has a clean calculus origin: differentiating $\mathrm{tr}(\Sigma_{\mathrm{eff}}(\pi))$ w.r.t.\ $\pi(x)$ yields $\mathrm{tr}(\Sigma_a(x))$, whereas differentiating $\log\det(\Sigma_{\mathrm{eff}}(\pi))$ via the identity $\frac{\diff}{\diff A}\log\det(A) = \mathrm{tr}(A^{-1}dA)$ yields $\mathrm{tr}(\Sigma_{\mathrm{eff}}^{-1} \Sigma_a(x))$. The $\Sigma_{\mathrm{eff}}^{-1}$ weighting introduces a global dependence on the full covariance structure: directions already well-estimated are upweighted, redirecting allocation toward reducing joint uncertainty across horizons. This is precisely what prevents $\pi^D$ from admitting a closed-form expression, since $\Sigma_{\mathrm{eff}}(\pi^D)$ depends on $\pi^D$ itself, making the characterization a fixed-point equation. \end{remark}

\textbf{Computation via fixed-point iteration:} Unlike the A-optimal rule in Proposition~\ref{prop:opt_assignment}, the D-optimal policy depends on $\Sigma_{\mathrm{eff}}(\pi^D)^{-1}$ and has therefore no closed-form solution. A simple solver is a fixed-point iteration:
\begin{enumerate}
\item Initialize $\pi^{(0)}$ (e.g., the A-optimal policy from Proposition~\ref{prop:opt_assignment}).
\item For $k = 0, 1, 2, \ldots$: \\
(a) Estimate $\Sigma_{\mathrm{eff}}(\pi^{(k)})$ using the accumulated data $\mathcal{H}_{r-1}$ or the population formula if available.\\
(b) Update 
\begin{align}
\tilde{\pi}^{(k+1)}(x) = \mathrm{clip}_\alpha\!\left(\frac{\sqrt{\mathrm{tr}\!\left(\Sigma_{\mathrm{eff}}(\pi^{(k)})^{-1} \Sigma_1(x)\right)}}{\sqrt{\mathrm{tr}\!\left(\Sigma_{\mathrm{eff}}(\pi^{(k)})^{-1} \Sigma_1(x)\right)} + \sqrt{\mathrm{tr}\!\left(\Sigma_{\mathrm{eff}}(\pi^{(k)})^{-1} \Sigma_0(x)\right)}}\right).
\end{align}\\
(c) Set $\pi^{(k+1)} \leftarrow \tilde{\pi}^{(k+1)}$ and stop when the iterates stabilize.
\end{enumerate}

\subsection{E-optimality}

E-optimality minimizes the largest eigenvalue of $\Sigma_{\mathrm{eff}}(\pi)$ and thereby controls the worst-case variance over all unit-norm linear contrasts $v^\top \tau$.

\begin{definition}[E-optimal allocation policy] 
\itshape 
A feasible allocation policy $\pi^E$ is \emph{E-optimal} if \begin{align}\pi^E \in \arg\min_{\pi \in \Pi_\alpha} \lambda_{\max}(\Sigma_{\mathrm{eff}}(\pi)),\end{align} where $\lambda_{\max}(\cdot)$ denotes the largest eigenvalue.
\end{definition}

\textbf{Interpretation:} By the variational characterization $\lambda_{\max}(\Sigma_{\mathrm{eff}}(\pi)) = \max_{\|v\|=1} v^\top \Sigma_{\mathrm{eff}}(\pi) v$, the E-optimal design minimizes the variance of the worst-case linear combination of survival effect curve contrasts $v^\top \tau$. This is the natural criterion when the experimenter lacks a prior preference over individual horizons $\tau_t$ and wishes to guard against the worst-case linear contrast.

\begin{theorem}[E-optimal allocation policy] 
Under the same conditions as Theorem~\ref{thm:d_optimal}, any E-optimal policy $\pi^E$ satisfies, for a.e.\ $x$, \begin{align}\pi^E(x) = \mathrm{clip}_\alpha\!\left(\frac{\sqrt{(v^\star)^\top \Sigma_1(x)\, v^\star}}{\sqrt{(v^\star)^\top \Sigma_1(x)\, v^\star} + \sqrt{(v^\star)^\top \Sigma_0(x)\, v^\star}}\right),\end{align} where $v^\star = v^\star(\pi^E) \in \mathbb{R}^{t_{\max}+1}$ is the unit eigenvector corresponding to $\lambda_{\max}(\Sigma_{\mathrm{eff}}(\pi^E))$. The characterization is a fixed-point equation since $v^\star$ depends on $\pi^E$ through $\Sigma_{\mathrm{eff}}(\pi^E)$.\end{theorem}

By the envelope theorem, at an interior optimum the first-order condition w.r.t.\ $\pi(x)$ yields $(v^\star)^\top [\partial \Sigma_{\mathrm{eff}}(\pi)/\partial \pi(x)] v^\star = 0$, which gives $-(v^\star)^\top \Sigma_1(x) v^\star / \pi(x)^2 + (v^\star)^\top \Sigma_0(x) v^\star / (1-\pi(x))^2 = 0$. Solving for $\pi(x)$ and applying the overlap constraint yields the stated expression.

\begin{remark} The quantity $(v^\star)^\top \Sigma_a(x) v^\star = \mathrm{Var}\bigl(\sum_t v^\star_t S_t(X,a)\,\xi_t(\mathcal{O},\eta_t) \mid X=x, A=a\bigr)$ is the conditional variance of the worst-case linear contrast $v^{\star\top}\phi(\mathcal{O};\tau,\eta)$ in arm $a$ at covariate value $x$. Thus, the E-optimal policy has the same Neyman form as the A- and D-optimal policies, but with the scalar $V_a(x)$ replaced by the worst-case directional variance $(v^\star)^\top \Sigma_a(x) v^\star$. \end{remark}

\textbf{Computation via fixed-point iteration:} The E-optimal policy can be computed analogously to the D-optimal case:
\begin{enumerate}
\item Initialize $\pi^{(0)}$ (e.g., the A-optimal policy from Proposition~\ref{prop:opt_assignment}).
\item For $k = 0, 1, 2, \ldots$: \\
(a) Estimate $\Sigma_{\mathrm{eff}}(\pi^{(k)})$ from accumulated data $\mathcal{H}_{r-1}$.\\
(b) Compute the leading eigenvector $v^{(k)} = \arg\max_{\|v\|=1} v^\top \Sigma_{\mathrm{eff}}(\pi^{(k)}) v$. \\
(c) Update 
\begin{align}\tilde{\pi}^{(k+1)}(x) = \mathrm{clip}_\alpha\!\left(\frac{\sqrt{(v^{(k)})^\top \Sigma_1(x)\, v^{(k)}}}{\sqrt{(v^{(k)})^\top \Sigma_1(x)\, v^{(k)}} + \sqrt{(v^{(k)})^\top \Sigma_0(x)\, v^{(k)}}}\right).
\end{align} \\
(d) Set $\pi^{(k+1)} \leftarrow \tilde{\pi}^{(k+1)}$ and stop when the iterates stabilize.
\end{enumerate}

\subsection{Comparison of A-, D-, and E-optimality}

All three criteria share the same Neyman allocation structure $\pi^\star(x) = \sqrt{Q_1(x)}/(\sqrt{Q_1(x)} + \sqrt{Q_0(x)})$ for an arm-specific scalar $Q_a(x)$, but differ in how they aggregate the cross-horizon covariance matrix $\Sigma_a(x)$: 
\begin{enumerate}
    \item A-optimality uses $Q_a(x) = \mathrm{tr}(\Sigma_a(x)) = V_a(x)$ (sum of marginal variances, closed-form); 
    \item D-optimality uses $Q_a(x) = \mathrm{tr}(\Sigma_{\mathrm{eff}}^{-1} \Sigma_a(x))$ (precision-weighted trace, fixed-point);
    \item E-optimality uses $Q_a(x) = (v^\star)^\top \Sigma_a(x) v^\star$ (worst-case directional variance, fixed-point).
\end{enumerate}
When $t_{\max}=0$, $v^\star = 1$ and, therefore, all three criteria coincide. The A-optimal criterion is the only one that separates additively over $t$, which is why it yields a closed-form solution; D- and E-optimal criteria couple all horizons jointly through $\Sigma_{\mathrm{eff}}$ and its eigenvector and therefore require iterative computation.

\newpage
\section{Algorithm}

\begin{algorithm}[H]
\caption{Adaptive Survival Estimator}
\label{alg:adaptive_survival}
\footnotesize
\begin{algorithmic}[1]
\Input Burn-in period $R_0$; initial policy $\pi_{\mathrm{init}}(X)$; truncation schedule $(k_r)_{r\geq 1}$; regression learners for $\lambda^S_t(X, a)$, $\lambda^G_t(X, a)$
\For{$r \gets 1$ \textbf{ to } $R$}
    \State Observe covariates $X_r$
    \If{$r \leq R_0$}
        \State Assign $A_r \sim \mathrm{Bern}(\pi_{\mathrm{init}}(X_r))$
    \Else
        \State Estimate nuisance functions $\hat\lambda^{S, \mathrm{cf}}_{t,r-1}(X, a)$, $\hat\lambda^{G, \mathrm{cf}}_{t,r-1}(X, a)$ from $\gH_{r-1}$ using cross-fitting
        \State Compute {$\widehat S^{\mathrm{cf}}_{t,r-1}(X,a) = \prod_{i=0}^{t}(1-\widehat\lambda^S_{i,r-1}(X,a))$} 
        \State Compute {$\widehat G^{\mathrm{cf}}_{t-1,r-1}(X,a) = \prod_{i=0}^{t-1}\left(1-\widehat\lambda^G_{i,r-1}(X,a)/ [1-\widehat\lambda^S_{i,r-1}(X,a)]\right)$}
        \State Compute $\hat V_{a,r-1}(X_r)$ for $a\in\{0,1\}$ using Eq.~\ref{eq:adap_v_a}
        \State Compute plug-in assignment probability:\\
        \qquad \quad {$\widetilde\pi_r(X_r \mid \gH_{r-1}) = \sqrt{\hat V_{1,r-1}(X_r)}/(\sqrt{\hat V_{0,r-1}(X_r)}+\sqrt{\hat V_{1,r-1}(X_r)})$}
        \State Apply truncation:\\
        \qquad \quad {$\pi_r(X_r \mid \gH_{r-1}) \leftarrow \min\!\left\{1-\tfrac{1}{k_r},\,\max\!\left\{\tfrac{1}{k_r},\,\widetilde\pi_r(X_r\mid\gH_{r-1})\right\}\right\}$}
        \State Assign $A_r \sim \mathrm{Bern}(\pi_r(X_r \mid \gH_{r-1}))$
    \EndIf
    \State Observe $\widetilde T_r$, $\Delta_r$
    \State Compute cross-fitted EIF pseudo-outcome $\hat{\bm{\phi}}_r^{\mathrm{cf}} = \bigl(\hat\phi_{0,r}^{\mathrm{cf}},\dots,\hat\phi_{t_{\max},r}^{\mathrm{cf}}\bigr)^\top$ using Eq.~\ref{eq:adaptive_cf_eif} with nuisances $\hat\eta_{t,r-1}^{(-J_r)}$
\EndFor
\vspace{-0.2cm}
\State \Return $\hat\tau_R^{\mathrm{ASE}} = \dfrac{1}{R}\displaystyle\sum_{r=1}^R \hat{\bm{\phi}}_r^{\mathrm{cf}}$
\end{algorithmic}
\end{algorithm}

\newpage
\section{Proofs}
\subsection{Proof of Theorem~\ref{thm:eff_bound}}
\label{app:proof_eff_bound}

The semiparametric efficiency bound for the vector parameter $\tau = (\tau_0, \dots, \tau_{\tmax})$ under a fixed allocation $\pi$ is the covariance of the efficient influence function
vector, i.e.,
\begin{align}
\Sigma_{\mathrm{eff}}(\pi) =
\Cov\bigl(\bm{\phi}(\gO;\tau,\eta_{\tmax})\bigr)
=
\E\!\left[
\bigl(\bm{\phi}(\gO;\tau,\eta_{\tmax})-\bm{\tau}\bigr)
\bigl(\bm{\phi}(\gO;\tau,\eta_{\tmax})-\bm{\tau}\bigr)^{\!\top}
\right],
\end{align}

So first, we decompose it at a discrete time $t$. To ease notation, we write $U_t=S_t(X,A)\,\xi(\gO,\eta_t)$, so that the centered EIF scores from Eq.~\ref{eq:eif} are given by
\begin{align}
\phi_t(\gO;\tau_t,\eta_t)
=
S_t(X, 1) - S_t(X, 0) - \underbrace{\frac{A-\pi(X)}{\pi(X)(1-\pi(X))}\,U_t}_{\Lambda_t} - \tau_t,
\end{align}

We show $\E[\xi(\gO,\eta_t)\mid X,A]=0$ first. Under Assumption~\ref{ass:survival}, we have $S_i(X,A)>0$ and $G_{i-1}(X,A)>0$, so each denominator is strictly positive and the summand is well defined. It therefore suffices to show that the numerator has a conditional mean zero for every $i$:
\begin{align}
\E\!\left[\mathbf 1\{\widetilde T=i,\Delta=1\}-\mathbf 1\{\widetilde T\ge i\}\,\lambda_i^S(X,A)\,\Big|\,X,A\right]
\;=\;0.
\label{eq:hazard_mds}
\end{align}

By the definition of the discrete-time event hazard, we yield
\begin{align}
\lambda_i^S(X,A)
\;=\;
\mathbb{P}\bigl(\widetilde T=i,\,\Delta=1 \,\big|\, \widetilde T\ge i,\,X,\,A\bigr).
\end{align}
Since $\{\widetilde T=i,\Delta=1\}\subseteq\{\widetilde T\ge i\}$, the event $\{\widetilde T=i,\Delta=1\}$ implies $\mathbf 1\{\widetilde T\ge i\}=1$, and therefore
\begin{align}
\E\!\left[\mathbf 1\{\widetilde T=i,\Delta=1\}\,\big|\,X,A\right]
&= \mathbb{P}\bigl(\widetilde T=i,\,\Delta=1\,\big|\,X,A\bigr)
\notag\\
&= \mathbb{P}\bigl(\widetilde T\ge i\,\big|\,X,A\bigr)\cdot
\mathbb{P}\bigl(\widetilde T=i,\,\Delta=1\,\big|\,\widetilde T\ge i,\,X,A\bigr)
\notag\\
&\;=\;\E\!\left[\mathbf 1\{\widetilde T\ge i\}\,\big|\,X,A\right]\cdot \lambda_i^S(X,A).
\end{align}
Rearranging yields Eq.~\ref{eq:hazard_mds}. Summing over $i=0,\dots,t$ and dividing each term by the $\sigma(X,A)$-measurable positive factor $S_i(X,A)\,G_{i-1}(X,A)$, we yield
\begin{align}
\E[\xi(\gO,\eta_t)\mid X,A]
\;=\;
\sum_{i=0}^{t}
\frac{\E\!\left[\mathbf 1\{\widetilde T=i,\Delta=1\}-\mathbf 1\{\widetilde T\ge i\}\,\lambda_i^S(X,A)\,\big|\,X,A\right]}{S_i(X,A)\,G_{i-1}(X,A)}
\;=\;0.
\end{align}

Then, we have $\E[U_t\mid X,A]=0$ and therefore
\begin{align}
\E[\Lambda_t\mid X]
=
\E\bigl[\E[\Lambda_t\mid X,A]\,\big|\,X\bigr]
=
0.
\end{align}

Thus, we can expand $V_{\mathrm{eff}, t}(\pi)$ as
\begin{align}
\Var(\phi_t) 
= \E[(\phi_t-\tau_t)^2]
= \E[(S_t(X, 1)-S_t(X, 0)-\tau_t)^2] + \E[\Lambda_t^2].
\end{align}

For the second term, the identity $\E[\Lambda_t^2] = \E[\E[\Lambda_t^2\mid X]]$ can be rewritten via
\begin{align}
&\E[\Lambda_t^2\mid X] \\
=& \E\left[ \left(\frac{A-\pi(X)}{\pi(X)(1-\pi(X))}\right)^{2} U_t^2 \mid X\right] \notag\\
=& \pi(X) \E\left[ \left(\frac{1}{\pi(X)}\right)^{2} U_t^2\mid X, A = 1 \right]
+ (1-\pi(X)) \E\left[ \left(\frac{1}{1-\pi(X)}\right)^{2} U_t^2 \mid X, A = 0 \right]\notag\\
=& \frac{1}{\pi(X)} \Var[S_t(X, 1)\xi(\gO, \eta_t) \mid X, A = 1 ] 
+ \frac{1}{1-\pi(X)}\Var[S_t(X, 0)\xi(\gO, \eta_t) \mid X, A = 0]. \notag
\end{align}

Taking the outer expectation, we have
\begin{align}
\Var(\phi_t)
=
\E\!\left[
\frac{v_{t,1}(X)}{\pi(X)}
+
\frac{v_{t,0}(X)}{1-\pi(X)}
\right]
+
\E[b_t(X)^2],
\end{align}
where $v_{t,a}(X):=\E[U_t^2\mid X,A=a]$, $b_t(X) = S_t(X, 1)-S_t(X, 0)-\tau_t$, and the second term is $\pi(X)$ independent.
\qed

\subsection{Proof of Proposition~\ref{prop:opt_assignment}}
\label{app:proof_opt_assignment}
Since $C$ in Theorem~\ref{thm:eff_bound} does not depend on $\pi$, we have
\begin{align}
\pi^\star
&=
\arg\min_{\pi}
\E\!\left[
\frac{V_1(X)}{\pi(X)}
+
\frac{V_0(X)}{1-\pi(X)}
\right]
\notag\\
&=
\arg\min_{\pi}
\E\!\left[
\E\!\left[
\frac{V_1(X)}{\pi(X)}
+
\frac{V_0(X)}{1-\pi(X)}
\,\Big|\, X
\right]
\right]
\notag\\
\Rightarrow\quad
\pi^\star(x)
&=
\arg\min_{p\in[\alpha,1-\alpha]}
\left(
\frac{V_1(x)}{p}
+
\frac{V_0(x)}{1-p}
\right),
\end{align}
where $V_a(x):=\sum_{t=0}^{t_{\max}} v_{t,a}(x)$.
The first-order condition gives
\begin{align}
\frac{V_0(x)}{(1-p)^2}
-
\frac{V_1(x)}{p^2}
=
0,
\end{align}
which yields $p=\sqrt{V_1(x)}/(\sqrt{V_1(x)}+\sqrt{V_0(x)})$. Applying the
overlap constraint, we yield
\begin{align}
\pi^\star(x)
=
\mathrm{clip}_\alpha\!\left(
\frac{\sqrt{V_1(x)}}{\sqrt{V_1(x)}+\sqrt{V_0(x)}}
\right).
\end{align}

\newpage
\subsection{Remark on necessity of stepwise overlap}
\label{app:remark_overlap}

\begin{remark}[Necessity of stepwise overlap]
Assumption~\ref{ass:uniform_overlap} imposes lower bounds on $1 - \lambda^S_i(x,a)$ and $1 - \lambda^G_i(x,a)/(1-\lambda^S_i(x,a))$ at every intermediate step $i \leq t$, rather than only at the endpoint $t$. This is necessary: since $G_{t-1}(x,a) = \prod_{i=0}^{t-1}\bigl(1 - \lambda^G_i(x,a)/(1-\lambda^S_i(x,a))\bigr)$ and $S_i(x,a) = \prod_{j=0}^{i}(1-\lambda^S_j(x,a))$ are both cumulative products, a single intermediate step at which $\lambda^S_i(x,a) \to 1$ drives $S_i(x,a) \to 0$ and hence $G_{i-1}(x,a)^{-1} \to \infty$, which the inverse probability of censoring weights in the EIF to become unbounded regardless of the behavior at other time steps. Assumption~\ref{ass:uniform_overlap} rules this out uniformly across $i \leq t$ and thus ensures that $S_i(x,a) \geq (c^S_t)^{t+1}$ and $G_{i-1}(x,a) \geq (c^G_t)^{t+1}$ for all $i \leq t$. Hence, this constitutes the discrete-time survival analogue of the positivity condition standard in semiparametric survival analysis~\citep{cui.2023, vanderlaan.2003, westling.2024}.
\end{remark}

\subsection{Proof of Theorem~\ref{thm:adaptive_asymptotic_normality_cf}}
\label{app:proof_thm_asymptotic}

The asymptotic argument relies on the martingale central limit theorem under a Lindeberg-type condition. We use a streamlined version of the martingale difference sequence (MDS) central limit theorem originally due to \citet{dvoretzky.1972}, as presented in Theorem 2 of \citet{zhang.2021a}.

To begin our proof, we define the centered EIF score at horizon $t$ as 
$\psi_{t, r} := \phi(X_r, A_r, \tilde{T}_r, \Delta_r; \tau_t, \pi_r, \hat{\eta}_{t, r})$, where $\phi$ is given in Eq.~\ref{eq:eif}. Let \\
$\eta_{t} = \{\lambda^S_{i}(X, 0), \lambda^S_{i}(X, 1), \lambda^G_{i}(X, 0), \lambda^G_{i}(X, 1)\}_{i=0}^t$ denote the oracle nuisance functions. We can decompose $\psi_{t, r}$ as 
\begin{align}\label{eq:decomp}
    \psi_{t, r} =& \phi_{t, r} - \tau_t + \hat{\phi}_{t, r} - \phi_{t, r}\\
    =& \underbrace{\phi(X_r, A_r, \tilde{T}_r, \Delta_r; \tau_t, \pi_r, \eta_t) - \tau_t}_{z_{t, r}} \notag\\
    &+ \underbrace{\phi(X_r, A_r, \tilde{T}_r, \Delta_r; \tau_t, \pi_r, \hat{\eta}_{t, r}) - \phi(X_r, A_r, \tilde{T}_r, \Delta_r; \tau_t, \pi_r, \eta_t)}_{m_{t, r}},
\end{align}
such that
\begin{align}
    \sqrt{R}\bigl(\hat{\tau}_{t, R}^{\mathrm{ASE}} - \tau_t\bigr) = \frac{1}{\sqrt{R}}\sum_{r=1}^R z_{t, r} + \frac{1}{\sqrt{R}}\sum_{r=1}^R m_{t, r}.
\end{align}

Our proof proceeds in three steps: (1) we prove that, for each horizon $t$, $\{z_{t, r}\}_{r=1}^R$ forms a martingale difference sequence, (2) we show that $\{z_{t, r}\}_{r=1}^R$ satisfy conditional variance convergence and Lindeberg condition, so we have the $\frac{1}{\sqrt{R}}\sum_{r=1}^R z_{t, r} \xrightarrow{d} \mathcal{N} (0, V_{\textrm{eff}, t}(\pi))$, and (3) we show that $\frac{1}{\sqrt{R}}\sum_{r=1}^R m_{t, r} = o_p(1)$. 

\subsubsection{Martingale difference sequence}
Before we start to the first steps, we state the martingale central limit theorem first.
\begin{theorem}[Martingale CLT, adapted from Theorem 2 in \citet{zhang.2021a}]
\label{thm:martingale_clt}
Let $\{z_{t, r}, \mathcal{H}_{r}\}_{r=1}^R$ be a real-valued sequence, and let $\bar{z}_{t, r} = \frac{1}{R}\sum_{r=1}^R z_{t, r}$ such that:
\begin{enumerate}
    \item (Martingale difference sequence) $\{z_{t, r}\}_{r = 1}^R$ is a martingale difference sequence; that is $\E[z_{t, r} \mid \mathcal{H}_{r-1}] = 0$ for all $r \in [1, R]$.
    \item (Conditional variance convergence) There exists a constant 
    $V_{\textrm{eff},t}(\pi) > 0$ such that 
    \begin{equation}
    \frac{1}{R}\sum_{r = 1}^R \E[z_{t, r}^2 \mid \mathcal{H}_{r-1}] \xrightarrow{p} V_{\textrm{eff},t}(\pi),
    \end{equation}
    \item (Lindeberg condition) For every $\varepsilon > 0$,
    \begin{equation}
    \frac{1}{R}\sum_{r = 1}^R \E \left[ z_{t, r}^2 \mathbbm{1}\{|z_{t, r}| > \varepsilon \sqrt{R}\} \mid \mathcal{H}_{r-1} \right] \xrightarrow{p} 0. 
    \end{equation}
\end{enumerate}
Then, $\sqrt{R}\bar{z}_{t, r} \xrightarrow{d}\mathcal{N} (0, V_{\textrm{eff},t}(\pi))$.
\end{theorem}

The first step is that we need to prove $\E[z_{t, r} \mid \mathcal{H}_{r-1}] = 0$:

\begin{align}
&\E[z_{t, r} \mid \mathcal{H}_{r-1}] \notag\\
=& \E[ S_t(X_r,1)-S_t(X_r,0) - \frac{A_r-\pi_r(X_r)}{\pi_r(X_r)\{1-\pi_r(X_r)\}}
\xi(\gH_r,\eta_t) S_t(X_r,A_r) \mid \mathcal{H}_{r-1}]  - \tau_t \notag\\
=& \tau_t - \tau_t
- \E\left[\E\left[\frac{A_r-\pi_r(X_r)}{\pi_r(X_r)\{1-\pi_r(X_r)\}}
\xi(\gH_r,\eta) S_t(X_r,A_r) \mid X_r, \mathcal{H}_{r-1}\right] \mid \mathcal{H}_{r-1} \right] \notag\\
=&- \E\left[\E\left[ \E\left[\frac{A_r-\pi_r(X_r)}{\pi_r(X_r)\{1-\pi_r(X_r)\}}
\xi(\gH_r,\eta) S_t(X_r,A_r) \mid A_r, X_r, \mathcal{H}_{r-1} \right]\mid X_r, \mathcal{H}_{r-1}\right] \mid \mathcal{H}_{r-1} \right]\notag\\
=& 0.
\end{align}

We further yield
\begin{align}
&\E\left[\xi(\gH_r,\eta_t)  \mid A_r, X_r, \mathcal{H}_{r-1}\right] \\
=& \E\left[ \sum_{i=0}^t \frac{\mathbbm{1}(\tilde{T}=i)\mathbbm{1}(\Delta = 1) - \mathbbm{1}(\tilde{T}\geq i)\lambda^s_i(X_r, A_r)}{S_i(X_r, A_r) G_{i-1}(X_r, A_r)}  \mid A_r, X_r, \mathcal{H}_{r-1}\right]\\
=& \sum_{i=0}^t \E\left[ \frac{\mathbbm{1}(\tilde{T}=i)\mathbbm{1}(\Delta = 1) - \mathbbm{1}(\tilde{T}\geq i)\lambda^s_i(X_r, A_r)}{S_i(X_r, A_r) G_{i-1}(X_r, A_r)}  \mid A_r, X_r, \mathcal{H}_{r-1}\right]\\
=& \sum_{i=0}^t 
\frac{\mathbb{P}\left(\tilde{T} = i, \Delta = 1 \mid X_r, A_r, \mathcal{H}_{r-1}\right)- \mathbb{P}\left(\tilde{T} \geq i \mid X_r, A_r, \mathcal{H}_{r-1} \right) \mathbb{P}\left(\tilde{T} = i, \Delta = 1 \mid \tilde{T} > i, X_r, A_r\right)}{S_i(X_r, A_r) G_{i-1}(X_r, A_r)}\\
=& \sum_{i=0}^t 
\frac{\mathbb{P}\left(\tilde{T} = i, \Delta = 1 \mid X_r, A_r, \mathcal{H}_{r-1}\right)- \mathbb{P}\left(\tilde{T} = i, \Delta = 1 \mid X_r, A_r, \mathcal{H}_{r-1}\right)}{S_i(X_r, A_r) G_{i-1}(X_r, A_r)}\\
=& 0.
\end{align}
Thus, for each fixed $t$, $\{z_{t,r},\mathcal H_r\}_{r\ge1}$ is a martingale difference sequence, since the assignment probability $\pi_r$ is $\mathcal H_{r-1}$-measurable and the oracle augmentation term satisfies $\E[\xi(\gH_r,\eta_t)\mid A_r,X_r,\mathcal H_{r-1}] = 0$.

\subsubsection{Conditional variance convergence and Lindeberg condition}
We first need to show that $\E[z_{t,r}^2\mid \mathcal H_{r-1}] - V_{\mathrm{eff},t}(\pi) \xrightarrow{p}0$. Let $b_t(X):=S_t(X,1)-S_t(X,0)-\tau_t$. Using $\E[\xi_t(\gH_r,\eta_t)\mid X_r,A_r,\mathcal H_{r-1}]=0$, we have
\begin{align}
\E[z_{t,r}^2\mid \mathcal H_{r-1}]
&=
\E\left[
b_t(X_r)^2
+
\frac{v_{t,1}(X_r)}{\pi_r(X_r)}
+
\frac{v_{t,0}(X_r)}{1-\pi_r(X_r)}
\;\middle|\;
\mathcal H_{r-1}
\right].
\end{align}
Therefore, we have
\begin{align}
&\left|
\E[z_{t,r}^2\mid \mathcal H_{r-1}]
-
V_{\mathrm{eff},t}(\pi)
\right|
\nonumber\\
&\le
\E\left[
v_{t,1}(X_r)
\left|
\frac{1}{\pi_r(X_r)}-\frac{1}{\pi(X_r)}
\right|
+
v_{t,0}(X_r)
\left|
\frac{1}{1-\pi_r(X_r)}-\frac{1}{1-\pi(X_r)}
\right|
\;\middle|\;
\mathcal H_{r-1}
\right]
\nonumber\\
\leq& 4\frac{k_r}{\varepsilon} \left(\frac{t+1}{\underline s_t\,\underline g_t}\right)^2 \|\pi_r - \pi\|_2 = o_p(1),
\end{align}

According to Assumption~\ref{ass:uniform_overlap}, we have $1-\lambda^S_i(x,a) \ge \underline{c}^S_t$ $1-\lambda^G_i(x,a) / (1-\lambda^S_i(x,a))\ge \underline{c}^G_t$, so we have $S_i(x, a) = \prod_{j=0}^i 1-\lambda^S_j(x,a) \geq (\underline{c}^S_t)^{i+1}\geq (\underline{c}^S_t)^{t+1} = \underline{s}_t$ and $G_{i-1}(x, a) = \prod_{j=0}^{i-1} 1- \lambda^G_j(x, a) / (1-\lambda^S_j(x,a)) \geq (\underline{c}^G_t)^{i}\geq (\underline{c}^G_t)^{t+1} = \underline{g}_t$. We note $\big| \sum_{i=0}^t \frac{S_i(X, a)}{S_i(X, a)G_{i-1}(X, a)}(\mathbf 1(\widetilde T=i,\Delta=1)-\mathbf 1(\widetilde T\ge i)\,\lambda_i^S(X,A))\big| \leq \sum_{i=0}^t \frac{S_i(X, a)}{S_i(X, a)G_{i-1}(X, a)} \leq \sum_{i=0}^t \frac{1}{\underline{s}_t \underline{g}_t}$. Therefore, $v_{t, a}(X) = \Var(\phi_t\mid X, A = a) \leq \E [\phi_t^2\mid X, A = a] \leq \left(\frac{t+1}{\underline s_t\,\underline g_t}\right)^2$ where the last line we use (i) the boundedness assumption from Assumption~\ref{ass:uniform_overlap} and thus $\|v_{t, 1}(X)\|_{\infty} \leq \left(\frac{t+1}{\underline s_t\,\underline g_t}\right)^2$, (ii) $\pi(X_r)$, $1-\pi(X_r) >\varepsilon$ from Theorem~\ref{thm:adaptive_asymptotic_normality_cf} statement, (iii) $\pi(X_r)$, $1-\pi(X_r) > 1/k_r$ by construction, and the $L_1$-norm is bounded by the $L_2$ norm. Thus, setting $\sigma^2 := V_{\textrm{eff},t}(\pi)$, we have that the term converges in probability to $\sigma^2$, i.e., $E[z_r^2 \mid \mathcal{H}_{r-1}] \xrightarrow{p} \sigma^2$, where $\sigma^2$ is finite. And we further have that $\frac{1}{R}\sum_{r = 1}^R \E[z_r^2 \mid \mathcal{H}_{r-1}] \xrightarrow{p} \sigma^2$.

Now, we prove the \textbf{Lindeberg condition}. For every $\varepsilon > 0$, by the truncation rule, $1/k_r\le \pi_r(X_r)\le 1-1/k_r$. Hence,
\begin{align}
\left|
\frac{A_r-\pi_r(X_r)}
{\pi_r(X_r)\{1-\pi_r(X_r)\}}
\right|
\le k_r.
\end{align}
Moreover, by Assumption~\ref{ass:uniform_overlap}, the survival and censoring weights in the oracle EIF are uniformly bounded for each fixed $t$. Therefore, there exists a constant $C_t<\infty$ such that $|z_{t,r}|\le C_t k_r$. Since $k_r=o(r^{1/4})$, we also have $k_r=o(\sqrt r)$. If $k_r$ is nondecreasing, then
\begin{align}
\max_{1\le r\le R}|z_{t,r}|
\le C_t k_R
=
o(\sqrt R).
\end{align}
Thus, for every $\varepsilon>0$, eventually $|z_{t,r}|\le \varepsilon\sqrt R$ for all $r\le R$, and hence $\mathbbm 1\{|z_{t,r}|>\varepsilon\sqrt R\}=0$ for all $r\le R$. Consequently,
\begin{align}
\frac{1}{R}\sum_{r=1}^R
\E\left[
z_{t,r}^2
\mathbbm 1\{|z_{t,r}|>\varepsilon\sqrt R\}
\mid \mathcal H_{r-1}
\right]
=0 ,
\end{align}
which proves the Lindeberg condition.
Thus, by Theorem~\ref{thm:martingale_clt} we proved that $\frac{1}{\sqrt{R}}\sum_{r=1}^R z_{t, r} \xrightarrow{d} \mathcal{N} (0, V_{\textrm{eff}, t}(\pi))$.

\subsubsection{The remainder term}
\label{app:proof_mt}

For a fixed horizon $t$, recall that
\begin{equation}
m_{t,r}
:=
\phi_t(\gH_r;\pi_r,\hat\eta_{t,r-1})
-
\phi_t(\gH_r;\pi_r,\eta_t).
\end{equation}
We aim to show
\begin{equation}
\frac{1}{\sqrt R}\sum_{r=1}^R m_{t,r}=o_p(1).
\end{equation}

We decompose
\begin{align}
\frac{1}{\sqrt R}\sum_{r=1}^R m_{t,r}
&=
\sqrt R\left(
\frac{1}{R}\sum_{r=1}^R
\E[m_{t,r}\mid \mathcal H_{r-1}]
\right)
+
\frac{1}{\sqrt R}\sum_{r=1}^R
\Bigl(
m_{t,r}-\E[m_{t,r}\mid \mathcal H_{r-1}]
\Bigr)
\notag\\
&:= \Delta_t^A+\Delta_t^B.
\end{align}

We show separately that $\Delta_t^A=o_p(1)$ and $\Delta_t^B=o_p(1)$.

\paragraph{Step 1: Bias term \texorpdfstring{$\Delta_t^A$}{DeltaA}.}
Let
\begin{equation}
\Delta_{t,r}^A:=\E[m_{t,r}\mid\mathcal H_{r-1}]
=
\E\!\left[
\phi_t(\gH_r;\pi_r,\hat\eta_{t,r-1})
-
\phi_t(\gH_r;\pi_r,\eta_t)
\;\middle|\;
\mathcal H_{r-1}
\right].
\end{equation}
Using the same expansion as in the proof of the robustness survival estimator, one obtains
\begin{align}
\Delta_{t,r}^A
=
\E\!\left[
\mathcal R_{t,r,1}(X_r)-\mathcal R_{t,r,0}(X_r)
\;\middle|\;
\mathcal H_{r-1}
\right],
\label{eq:deltaA-remainder}
\end{align}
where, for $a\in\{0,1\}$,
\begin{align}
\mathcal R_{t,r,a}(x)
=&
\sum_{i=0}^{t}
\frac{\hat S_{t,r-1}(x,a)}
{\hat S_{i,r-1}(x,a)\hat G_{i-1,r-1}(x,a)}
S_i(x,a)\,
\bigl\{\hat\lambda_{i,r-1}^S(x,a) - \lambda_i^S(x,a)\bigr\}
\bigl\{\hat G_{i-1,r-1}(x,a)-G_{i-1}(x,a)\bigr\}\notag\\
\label{eq:Rtra}
\end{align}

According to assumptions in Theorem~\ref{thm:adaptive_asymptotic_normality_cf}, for each arm $a$, we have $\max_{0\leq i \leq \tmax} \|\hat{\lambda}^S_{i, r-1}(a, \cdot)-\lambda^S_i(a, \cdot)\|_2 =o_p(r^{-1/4})$ and $\max_{0\leq i \leq \tmax-1} \|\hat{\lambda}^G_{i, r-1}(a, \cdot)-\lambda^G_i(a, \cdot)\|_2 =o_p(r^{-1/4})$. This implies 
\begin{align}
\label{eq:sum_rate}
&\left(
\sum_{i=0}^{t}
\|\hat\lambda^S_{i,r-1}(a,\cdot)-\lambda_i^S(a,\cdot)\|_2
\right)^2
=
o_p(r^{-1/2}),\\
&\left(
\sum_{i=0}^{t}
\|\hat\lambda^S_{i,r-1}(a,\cdot)-\lambda_i^S(a,\cdot)\|_2
\right)
\left(
\sum_{i=0}^{t-1}
\|\hat\lambda^G_{i,r-1}(a,\cdot)-\lambda_i^G(a,\cdot)\|_2
\right)
=
o_p(r^{-1/2}).
\end{align}

Then, the prefactor in Eq.\ref{eq:Rtra} is uniformly bounded by a constant $C_t<\infty$ depending only on $t,\underline s_t,\underline g_t$. Hence
\begin{align}
|\Delta_{t,r}^A|
&\le
C_t \sum_{a\in\{0,1\}}
\left[
\left(
\sum_{i=0}^{t}
\|\hat\lambda_{i,r-1}^S(a,\cdot)-\lambda_i^S(a,\cdot)\|_2
\right)^2
\right.
\nonumber\\
&\qquad\left.
+
\left(
\sum_{i=0}^{t}
\|\hat\lambda_{i,r-1}^S(a,\cdot)-\lambda_i^S(a,\cdot)\|_2
\right)
\left(
\sum_{i=0}^{t-1}
\|\hat\lambda_{i,r-1}^G(a,\cdot)-\lambda_i^G(a,\cdot)\|_2
\right)
\right].
\label{eq:deltaA-bound}
\end{align}
By the assumed product-rate condition, the right-hand side is $o_p(r^{-1/2})$. Therefore, we have
\begin{equation}
\Delta_{t,r}^A=o_p(r^{-1/2}),
\qquad\text{and hence}\qquad
\Delta_t^A
=
\sqrt R\left(\frac1R\sum_{r=1}^R \Delta_{t,r}^A\right)
=o_p(1).
\end{equation}

\paragraph{Step 2: Martingale fluctuation term $\Delta_t^B$.}
Define $ U_{t,r} := m_{t,r}-\E[m_{t,r}\mid\mathcal H_{r-1}]$.
Then $\{U_{t,r},\mathcal H_r\}_{r\ge1}$ is a martingale difference sequence, and $\Delta_t^B = \frac{1}{\sqrt R}\sum_{r=1}^R U_{t,r}$.
It therefore suffices to show that
\begin{equation}
\Var(\Delta_t^B)=o(1).
\end{equation}
Since the cross-terms vanish for a martingale difference sequence, we have
\begin{align}
\Var(\Delta_t^B)
&=
\frac{1}{R}\sum_{r=1}^R \E[U_{t,r}^2]
\le
\frac{1}{R}\sum_{r=1}^R
\E\!\left[
m_{t,r}^2
\right].
\label{eq:deltaB-var-start}
\end{align}

We now bound $m_{t,r}^2$. We write
\begin{equation}
\hat\xi_{t,r}:=\xi_t(\gH_r,\hat\eta_{t,r-1}),
\qquad
\xi_{t,r}:=\xi_t(\gH_r,\eta_t),
\qquad
W_r:=\frac{A_r-\pi_r(X_r)}{\pi_r(X_r)\{1-\pi_r(X_r)\}},
\end{equation}
and
\begin{equation}
\Delta_{t,r}^S(a):=\hat S_{t,r-1}(X_r,a)-S_t(X_r,a),\qquad a\in\{0,1\}.
\end{equation}
Then, we yield
\begin{align}
m_{t,r}
&=
\Delta_{t,r}^S(1)-\Delta_{t,r}^S(0)
-
W_r\Bigl[
(\hat\xi_{t,r}-\xi_{t,r})\hat S_{t,r-1}(X_r,A_r)
+
\xi_{t,r}\,\Delta_{t,r}^S(A_r)
\Bigr].
\end{align}
By $(u_1+\cdots+u_4)^2\le 4(u_1^2+\cdots+u_4^2)$, we have
\begin{align}
m_{t,r}^2
\le
C_tk_r^2\Bigl(
|\Delta_{t,r}^S(1)|^2
+
|\Delta_{t,r}^S(0)|^2
+
|\hat\xi_{t,r}-\xi_{t,r}|^2
\Bigr),
\label{eq:mt-basic-bound}
\end{align}
where $C_t<\infty$ depends only on $t$ and the overlap constants $\underline s_t,\underline g_t$.

It remains to bound the two ingredients on the right-hand side.

\paragraph{Step 2a: Control of \texorpdfstring{$\Delta_{t,r}^S(a)$}{DeltaS}.}
Since
\begin{equation}
S_t(x,a)=\prod_{i=0}^t (1-\lambda_i^S(x,a)),
\qquad
\hat S_{t,r-1}(x,a)=\prod_{i=0}^t (1-\hat\lambda_{i,r-1}^S(x,a)),
\end{equation}
the product-difference identity gives
\begin{align}
|\Delta_{t,r}^S(a)|
\le
\sum_{i=0}^t
|\hat\lambda_{i,r-1}^S(X_r,a)-\lambda_i^S(X_r,a)|.
\end{align}
Hence, by the Cauchy--Schwarz inequality, we yield
\begin{align}
\E\!\left[
|\Delta_{t,r}^S(a)|^2
\;\middle|\;
\mathcal H_{r-1}
\right]
\le
(t+1)\sum_{i=0}^t
\E\!\left[
\bigl(\hat\lambda_{i,r-1}^S(X_r,a)-\lambda_i^S(X_r,a)\bigr)^2
\;\middle|\;
\mathcal H_{r-1}
\right].
\label{eq:deltaS-bound}
\end{align}

\paragraph{Step 2b: Control of \texorpdfstring{$\hat\xi_{t,r}-\xi_{t,r}$}{hatxi-xi}.}
Using the definition of $\xi_t$ and a first-order decomposition of the difference between
\begin{equation}
\frac{N_{t,i}-Y_{t,i}\hat\lambda_{i,r-1}^S}{\hat S_{i,r-1}\hat G_{i-1,r-1}}
\qquad\text{and}\qquad
\frac{N_{t,i}-Y_{t,i}\lambda_i^S}{S_iG_{i-1}},
\end{equation}
together with the uniform lower bounds $\underline s_t,\underline g_t$, we obtain
\begin{align}
|\hat\xi_{t,r}-\xi_{t,r}|^2
\le
C_t
\sum_{a\in\{0,1\}}
\left\{
\sum_{i=0}^{t}
\bigl(\hat\lambda_{i,r-1}^S(X_r,a)-\lambda_i^S(X_r,a)\bigr)^2
+
\sum_{i=0}^{t-1}
\bigl(\hat\lambda_{i,r-1}^G(X_r,a)-\lambda_i^G(X_r,a)\bigr)^2
\right\},
\label{eq:xi-difference-bound}
\end{align}
for some constant $C_t<\infty$.

Combining Eq.~\ref{eq:mt-basic-bound}, Eq.~\ref{eq:deltaS-bound}, and Eq.~\ref{eq:xi-difference-bound}, we get
\begin{align}
\E[m_{t,r}^2\mid\mathcal H_{r-1}]
\le
C_t k_r^2
\sum_{a\in\{0,1\}}
\left\{
\sum_{i=0}^{t}
\|\hat\lambda_{i,r-1}^S(\cdot,a)-\lambda_i^S(\cdot,a)\|_2^2
+
\sum_{i=0}^{t-1}
\|\hat\lambda_{i,r-1}^G(\cdot,a)-\lambda_i^G(\cdot,a)\|_2^2
\right\}.
\label{eq:mt-second-moment-final}
\end{align}
Under the stated $L_2$ rate conditions and the boundedness implied by Assumption~\ref{ass:uniform_overlap}, the expectation of the right-hand side of Eq.~\ref{eq:mt-second-moment-final} is $o(1)$.

\paragraph{Conclusion.}
Since both $\Delta_t^A=o_p(1)$ and $\Delta_t^B=o_p(1)$, we conclude that $\frac{1}{\sqrt R}\sum_{r=1}^R m_{t,r}=o_p(1)$.
This proves that the nuisance remainder is asymptotically negligible.

If the limiting policy satisfies \(\pi=\pi^\star\), then the limiting variance in the above display becomes \(V_{\mathrm{eff},t}(\pi^\star)\). By Theorem~\ref{thm:eff_bound}, this is the semiparametric efficiency bound under the A-optimal allocation.

\subsection{Proof of Proposition~\ref{prop:policy_convergence}}
\label{app:proof_policy_convergence}

By Eq.~\ref{eq:adap_v_a}, $V_a(x) = \sum_{t=0}^{t_{\max}} S_t(x,a)^2 \sum_{i=0}^{t} \lambda^S_i(x,a) / (S_i(x,a)\, G_{i-1}(x,a))$ is a Lipschitz function of $(\bm{\lambda}^S(X,a), \bm{\lambda}^G(X,a))$ on the region where $S_i(x,a)$ and $G_{i-1}(x,a)$ are bounded away from zero; Assumption~\ref{ass:uniform_overlap} guarantees precisely this uniform lower bound. Hence, the $o_p(r^{-1/4})$ hazard consistency implies $\|\hat{V}_{a,r-1} - V_a\|_2 = o_p(r^{-1/4})$ for each $a \in \{0,1\}$ by Lipschitz composition. Furthermore, Assumption~\ref{ass:uniform_overlap} implies $S_t(x,a) \geq (\underline{c}^S_t)^{t+1} > 0$ and $G_{i-1}(x,a) \geq \prod_{j} (1 - \lambda^G_j / (1-\lambda^S_j)) \cdot (1-\lambda^S_j) > 0$ uniformly, so that the denominator $\sqrt{V_0(x)} + \sqrt{V_1(x)}$ in the plug-in policy is bounded away from zero. The map $(V_0, V_1) \mapsto \sqrt{V_1}/(\sqrt{V_0}+\sqrt{V_1})$ is therefore Lipschitz, giving $\|\widetilde\pi_r - \pi^\star\|_2 = O_p(r^{-1/4})$ by Lipschitz composition. Since $\pi^\star(X) \in [\alpha, 1-\alpha]$ and $k_r \to \infty$, we have $1/k_r < \alpha$ for all sufficiently large $r$, so $\pi^\star(X) \in [1/k_r, 1-1/k_r]$. Because clipping onto $[1/k_r, 1-1/k_r]$ is the Euclidean projection and is non-expansive with respect to any point inside the interval, we have $|\pi_r(x) - \pi^\star(x)| \leq |\widetilde\pi_r(x) - \pi^\star(x)|$ pointwise. Hence, we arrive at $\|\pi_r - \pi^\star\|_2 \leq \|\widetilde\pi_r - \pi^\star\|_2 = O_p(r^{-1/4})$. Since $k_r = o(r^{1/4})$, it follows that $k_r\|\pi_r - \pi^\star\|_2 = o_p(1)$. 
\qed

\newpage
\subsection{Proof of Corollary~\ref{cor:orth_lambdaG}}
\label{app:orthogonality}

The claim is an immediate consequence of Theorem~\ref{thm:adaptive_asymptotic_normality_cf} and Eq.~\ref{eq:sum_rate}, which yields
\begin{align}
&\bigl|\hat\tau_{t,R}^{\mathrm{ASE}}-\tau_t\bigr|
=O_p(R^{-1/2})\\
&+O_p\left(
(\sum_{i=0}^{t} \|\hat{\lambda}^S_{i, r-1}(a, \cdot)-\lambda^S_i(a, \cdot)\|_2 )^2
+ \sum_{i=0}^{t}\|\hat{\lambda}^S_{i, r-1}(a, \cdot)-\lambda^S_i(a, \cdot)\|_2\sum_{i=0}^{t-1}\|\hat{\lambda}^G_{i, r-1}(a, \cdot) - \lambda^G_i(a, \cdot)\|_2
\right).\notag
\end{align}
Since $\left(\sum_{i=0}^{t} \|\hat{\lambda}^S_{i, r-1}(a, \cdot)-\lambda^S_i(a, \cdot)\|_2\right)^2=o_p(r^{-1/2})$ and $\sum_{i=0}^{t}\|\hat{\lambda}^S_{i, r-1}(a, \cdot)-\lambda^S_i(a, \cdot)\|_2\sum_{i=0}^{t-1}\|\hat{\lambda}^G_{i, r-1}(a, \cdot) - \lambda^G_i(a, \cdot)\|_2 = o_p(r^{-1/2})$, so $\hat\tau_{t,R}^{\mathrm{ASE}}\xrightarrow{p}\tau_t$. The absence of any standalone term \(O_p(\|\hat\lambda_t^G-\lambda_t^G\|_2)\) shows that perturbations in $\lambda^G$ enter only through interaction terms, which is exactly the robustness condition we aimed to show.
\qed

\newpage
\section{Extensions to no ties}
\label{app:no-ties}
In the no-ties setting, we refine the observed data structure by recording separate event and censoring indicators. Specifically, each observation is $\gO=(X,A,\widetilde T,\Delta^S,\Delta^G)$, where $\widetilde T=\min(T,C)$, $\Delta^S=\mathbf 1(T=\widetilde T)$, $\Delta^G=\mathbf 1(C=\widetilde T$), so that $\Delta^S+\Delta^G=1$ almost surely. In this case, the event and censoring hazards are defined separately by
\begin{align}
\lambda_t^S(x,a)
&=
\mathbb P(T=t \mid T\ge t, X=x,A=a),
\\
\lambda_t^G(x,a)
&=
\mathbb P(C=t \mid C\ge t, X=x,A=a),
\end{align}
with corresponding survival functions
\begin{align}
&S_t(x,a)=\prod_{i=0}^t (1-\lambda_i^S(x,a)),
&G_{t-1}(x,a)=\prod_{i=0}^{t-1} (1-\lambda_i^G(x,a)).
\end{align}
For a possibly misspecified nuisance limit $\tilde\eta_t=(\tilde\lambda_i^S(\cdot,\cdot),\tilde\lambda_i^G(\cdot,\cdot))_{i=0}^t$, 
we analogously define
\begin{align}
&\tilde S_t(x,a)
=
\prod_{i=0}^t (1-\tilde\lambda_i^S(x,a)),
&\tilde G_t(x,a)
=
\prod_{i=0}^{t-1} (1-\tilde\lambda_i^G(x,a)).
\end{align}

In the no-ties setting, the censoring survival function \(G_t\) depends only on the censoring hazard \(\lambda^G\), so the event and censoring nuisances form two variation-independent blocks. As a result, the adaptive survival estimator satisfies a literal two-block double robustness property.

\begin{theorem}[Convergence rate of the adaptive survival estimator under no ties]
\label{thm:convergence_rate_no_ties}
Suppose Assumptions~\ref{ass:causal} and~\ref{ass:survival} hold, and that there exists a non-adaptive policy \(\pi(X)\in[\varepsilon,1-\varepsilon]\) for some \(\varepsilon>0\) such that
\begin{align}
k_r\|\pi_r-\pi\|_2=o_p(1).
\end{align}
Let \(\tilde\eta_t\) denote a possibly misspecified limit of the cross-fitted nuisance estimators, and suppose
\begin{align}
k_r\|\hat f_{t,r-1}-\tilde f_t\|_2=o_p(1)
\end{align}
componentwise for \(\tilde f_t\in\tilde\eta_t\). Then, for each horizon \(t\le t_{\max}\),
\begin{align}
\bigl|\hat\tau_{t}^{\mathrm{ad\text{-}cf}} - \tau_t\bigr|
&=
O_p(R^{-1/2})
+
O_p\!\left(\|\tilde S_t - S_t\|_2\,\|\tilde G_t - G_t\|_2\right)
+
\sum_{i=0}^t
O_p\!\left(
\|\tilde\lambda^S_i - \lambda^S_i\|_2\,
\|\tilde G_{i-1} - G_{i-1}\|_2
\right).
\end{align}
\end{theorem}

\begin{corollary}[Double robustness under no ties]
\label{cor:double_robust_no_ties}
Under the conditions of Theorem~\ref{thm:convergence_rate_no_ties}, if either
\begin{align}
\|\hat\lambda^S_{t,r-1}-\lambda_t^S\|_2=o_p(1)
\qquad \text{for all } t,
\end{align}
or
\begin{align}
\|\hat\lambda^G_{t,r-1}-\lambda_t^G\|_2=o_p(1)
\qquad \text{for all } t,
\end{align}
then
\begin{align}
\hat\tau_{R,t}^{\mathrm{ad\text{-}cf}} \overset{p}{\longrightarrow} \tau_t
\qquad \text{for each } t\le t_{\max}.
\end{align}
\end{corollary}

\begin{proof}
The claim follows directly from Theorem~\ref{thm:convergence_rate_no_ties}. If either the event hazard estimators or the censoring hazard estimators are \(L_2\)-consistent, then the corresponding product terms in the remainder are \(o_p(1)\), and hence $\bigl|\hat\tau_{t}^{\mathrm{ad\text{-}cf}} - \tau_t\bigr| = o_p(1)$.
Therefore, $\hat\tau_{R,t}^{\mathrm{ad\text{-}cf}} \overset{p}{\longrightarrow} \tau_t$ for each, $t\le t_{\max}$.
\end{proof}

\textbf{Why the tied-case remainder is not of pure product form:}
\label{app:tied-remainder}

We explain the structural reason why the second-order remainder in the tied setting, stated in Theorem~\ref{thm:adaptive_asymptotic_normality_cf}, contains both a pure event-hazard quadratic term and an event--censoring cross-interaction term, in contrast to the clean two-block product structure of the no-ties case (Appendix~\ref{app:no-ties}).

\textbf{The source of changes:}
In both settings, the key remainder after first-order cancellation involves, for each time horizon $t$ and arm $a$, a sum of the form
\begin{align}
\mathrm{Rem}_t
:=
\sum_{i=0}^{t}
\bigl(\hat\lambda_i^S - \lambda_i^S\bigr)
\cdot
\bigl(\hat G_{i-1} - G_{i-1}\bigr).
\label{eq:rem-structure}
\end{align}
This cross-product arises because the EIF augmentation term at each time $i$ is weighted by $1/(S_i \cdot G_{i-1})$: when the event hazard estimator deviates from truth, the error in $\hat G_{i-1}$ multiplies it. The two settings differ precisely in how $G_{i-1}$ depends on the nuisance parameters.

\textbf{Tied setting:}
In the tied setting, the censoring survival function is constructed via a cause-specific factorization, i.e.,
\begin{align}
G_{i-1}(x,a)
=
\prod_{k=0}^{i-1}\left(1 - \frac{\lambda_k^G(x,a)}{1-\lambda_k^S(x,a)}\right),
\end{align}
because at each discrete time $k$ the conditional probability of censoring given survival to time $k$ is $\lambda_k^G/(1-\lambda_k^S)$, not $\lambda_k^G$ alone. Let $\delta_k^S := \hat\lambda_k^S - \lambda_k^S$ and $\delta_k^G := \hat\lambda_k^G - \lambda_k^G$. A first-order Taylor expansion of $\hat G_{i-1} - G_{i-1}$ gives
\begin{align}
\hat G_{i-1} - G_{i-1}
\approx
-G_{i-1}
\sum_{k=0}^{i-1}
\frac{1}{1 - \lambda_k^G/(1-\lambda_k^S)}
\left[
\frac{\delta_k^G}{1-\lambda_k^S}
+
\frac{\lambda_k^G\,\delta_k^S}{(1-\lambda_k^S)^2}
\right].
\end{align}
The critical observation is the second term inside the brackets: because $\lambda_k^G$ appears in the numerator of the cause-specific hazard, a perturbation $\delta_k^S$ in the event hazard also perturbs $G_{i-1}$. Substituting into Eq.~\ref{eq:rem-structure} and bounding in $L_2$, the remainder can be rewritten as
\begin{align}
\|\mathrm{Rem}_t\|
&\lesssim
\underbrace{
\sum_{i=0}^{t}\|\delta_i^S\|_2
\sum_{k=0}^{i-1}
\frac{\lambda_k^G}{(1-\lambda_k^S)^2}\|\delta_k^S\|_2
}_{\text{pure }\lambda^S\text{ quadratic}}
+
\underbrace{
\sum_{i=0}^{t}\|\delta_i^S\|_2
\sum_{k=0}^{i-1}
\frac{\|\delta_k^G\|_2}{1-\lambda_k^S}
}_{\lambda^S\text{--}\lambda^G\text{ cross-interaction}}.
\end{align}
Since $t \le t_{\max}$ is fixed, summing over $i$ and $k$ contributes only a finite multiplicative constant, so, at the level of asymptotic rates, this gives
\begin{align}
\|\mathrm{Rem}_t\|
=
O_p\!\left(
\Bigl(\sum_{i=0}^{t}\|\hat\lambda_i^S - \lambda_i^S\|_2\Bigr)^2
+
\Bigl(\sum_{i=0}^{t}\|\hat\lambda_i^S - \lambda_i^S\|_2\Bigr)
\Bigl(\sum_{i=0}^{t-1}\|\hat\lambda_i^G - \lambda_i^G\|_2\Bigr)
\right).
\end{align}
This is precisely the remainder structure appearing in Theorem~\ref{thm:adaptive_asymptotic_normality_cf}. Note in particular that there is no standalone term of the form $O_p(\sum_i\|\delta_i^G\|_2^2)$: a perturbation in $\lambda^G$ alone, without any event-hazard error, does not appear in $\mathrm{Rem}_t$ at all, which is the robustness of the estimating equation with respect to the censoring nuisance.

\textbf{No-ties setting (contrast):}
In the no-ties setting, the censoring survival function factorizes independently of the event hazard:
\begin{align}
G_{i-1}(x,a)
=
\prod_{k=0}^{i-1}(1-\lambda_k^G(x,a)).
\end{align}
The first-order expansion then gives
\begin{align}
\hat G_{i-1} - G_{i-1}
\approx
-G_{i-1}\sum_{k=0}^{i-1}\frac{\delta_k^G}{1-\lambda_k^G},
\end{align}
which involves only $\delta^G$ and contains no $\delta^S$ term. Consequently, $\mathrm{Rem}_t$ involves only the cross-product $\delta^S \cdot \delta^G$, yielding the clean remainder structure of Theorem~\ref{thm:convergence_rate_no_ties}:
\begin{align}
\|\mathrm{Rem}_t\|
=
O_p\!\left(
\sum_{i=0}^{t}
\|\tilde\lambda_i^S - \lambda_i^S\|_2\,
\|\tilde G_{i-1} - G_{i-1}\|_2
\right).
\end{align}
Because the two nuisance blocks are now variation-independent, the remainder vanishes whenever either block is $L_2$-consistent, giving the literal two-block double robustness of Corollary~\ref{cor:double_robust_no_ties}. In the tied setting, this property fails: even if $\lambda^G$ is estimated perfectly, a nonzero $\delta^S$ still induces a nonzero $\hat G_{i-1} - G_{i-1}$ through the cause-specific factorization, so the pure $\lambda^S$ quadratic term remains.

\newpage
\section{Asymptotic confidence sequences}
\label{app:asc}

The fixed-time intervals in Section~\ref{sec:theory} guarantee $(1-\alpha)$ coverage \emph{solely} at pre-specified sample size $R$. In practice, however, analysts often \emph{peek} at interim results and may stop the study early once a decision rule is met~\citep {ramdas.2023}; this behavior invalidates fixed-time intervals. To remain valid under such data-dependent stopping, one needs a \emph{confidence sequence (CS)}, i.e., a collection of intervals $[L_{t, r}, U_{t, r}]_{r\geq 1}$ at horizon $t$ satisfying the time-uniform guarantee, $\mathbb{P}(\forall r\in \mathbb{N}^{+}: \tau_t \in [L_{t, r}, U_{t, r}]) \geq 1-\alpha$.

Constructing \emph{non-asymptotic}, anytime-valid CSs can be difficult when the target estimand contains estimated nuisance functions. Fortunately, for the adaptive survival estimator, the nuisance-induced remainder is asymptotically negligible under the rate conditions below, so an \emph{asymptotic CS} (valid after a finite, burn-in phase) remains both tractable and practically useful.

\begin{definition}[Asymptotic time-uniform coverage~\citep{dalal.2024}]
\label{def:acs}
\itshape
A sequence of random intervals $\widetilde{C}_{t, r} = [\widetilde{L}_{t, r}, \widetilde{U}_{t, r}]_{t\geq1}$ is an asymptotic time uniform $(1-\alpha)$ confidence sequence (AsympCS) for a parameter $\tau_t$ at horizon $t$ if the following two conditions hold:
\begin{enumerate}
    \item \textbf{\textit{Asymptotic confidence sequence:}} there exists an exact (potentially unknown) $(1-\alpha)$ confidence sequence $C^\star_{t, r} = [L^\star_{t, r}, U^\star_{t, r}]_{t\geq1}$ such that $\widetilde{L}_{t, r}/L^\star_{t, r} \rightarrow 1$ and $\widetilde{U}_{t, r}/U^\star_{t, r} \rightarrow 1$ almost surely.
    \item \textbf{\textit{Asymptotic time-uniform coverage:}} $\lim_{R_0 \rightarrow \infty} \mathbb{P}(\forall r \geq R_0: \tau_t \in \widetilde{C}_{t, r}) \geq 1-\alpha$.
\end{enumerate}
\end{definition}
Definition~\ref{def:acs} can be read as follows: if one waits to ``peek'' until the sample size is sufficiently large ($R\geq R_0$ for some burn-in $R_0$), the band then covers the true parameter \textit{every} later time with probability approaching $1-\alpha$. In practice, rare coverage failures occur almost exclusively during this short initial window; later, the intervals tighten rapidly and deliver appreciable power gains over fully non-asymptotic sequences~\citep{cook.2024, waudby-smith.2024}.

Building on the methodologies of \citet{waudby-smith.2024} and \citet{cook.2024}, we now present the corresponding asymptotic confidence sequence (AsympCS) for our estimator:
\begin{theorem}[AsympCS for adaptive survival estimator]
Suppose Assumptions~\ref{ass:causal}, \ref{ass:survival} and~\ref{ass:uniform_overlap} hold and that there exists a non-adaptive policy $\pi(X)\in [\varepsilon, 1-\varepsilon]$ for some $\varepsilon>0$ such that the adaptive assignment policy $\pi_r(X\mid \gH_{r-1})$ are $L_2$-consistent relative to truncation schedule, i.e. $k_r \| \pi_{r}-\pi\|_2 = o_{a.s.}(1)$. Furthermore, assume that, for each fixed $t$ and treatment $a$, $\max_{0\leq i \leq \tmax}\|\hat{\lambda}^S_{i, r-1}(a, \cdot)-\lambda^S_i(a, \cdot) \|_2 = o_{a.s.}\left(\left(\frac{\log r}{r}\right)^{1/4}\right)$, and $\max_{0\leq i \leq \tmax-1}\|\hat{\lambda}^G_{i, r-1}(a, \cdot)-\lambda^G_i(a, \cdot) \|_2 = o_{a.s.}\left(\left(\frac{\log r}{r}\right)^{1/4}\right)$. Let 
\begin{align}
    \hat{V}_{t,R}:=\frac{1}{R}\sum_{r=1}^R \left( \hat\phi_{t,r} - \hat\tau_{t,R}^{\mathrm{ASE}} \right)^2,
\end{align}
be the estimated variance of $\{\phi(O_r; \pi_r, \hat{\eta}_{t, r})\}$, and fix a user-specified $\rho > 0$. Then, for all $R \ge 1$, the interval
\begin{align}\label{eq:cs}
    \tilde{C}_R^{\mathrm{AsympCS}} := \left(\hat{\tau}_{t, R}^{\mathrm{ASE}} \;\pm\; \sqrt{\frac{2(R\hat{V}_{t, R}\rho^2 + 1)}{R^2\rho^2}\,\log\!\left(\frac{\sqrt{R\hat{V}_{t, R}\rho^2 + 1}}{\alpha}\right)}\right)
\end{align}
is an asymptotic $(1-\alpha)$-CS for $\tau_t$. The width is approximately minimized at
    \begin{equation}
        \rho^\star = \sqrt{\frac{-2\log\alpha + \log(-2\log\alpha + 1)}{R}}.
    \end{equation}
\end{theorem}

\begin{remark}[Difference in convergence rates for fixed-time and anytime-valid inference.]
For fixed-time inference, the second-order nuisance remainder is required to be $o_p(r^{-1/2})$, which is implied by pointwise nuisance rates $o_p(r^{-1/4})$. For the AsympCS, the asymptotic linearity remainder only needs to be $o_{a.s.}(\sqrt{\log r/r})$, which is implied by pointwise nuisance rates $o_{a.s.}((\log r/r)^{1/4})$. Thus, the AsympCS condition allows a slightly slower rate up to a logarithmic factor, but requires almost-sure rather than in-probability convergence.
\end{remark}

\begin{proof}
For the proof, we rely on existing results for asymptotic confidence sequences~\citep{waudby-smith.2024}. For completeness, we provide this in Lemma~\ref{lem:ws}

\begin{lemma}[Corollary 3.4 of \citet{waudby-smith.2024}]\label{lem:ws}
    Suppose $\hat{\tau}_R$ is an asymptotically linear estimator with influence function $\phi$ satisfying
    \begin{equation}\label{eq:al}
        \hat{\tau}_{t, R} - \tau = \frac{1}{R}\sum_{r=1}^{R} \phi(O_r;\, \pi_r, \eta_{t, r}) + o_{a.s.}\!\left(\sqrt{\frac{\log R}{R}}\right).
    \end{equation}
    If $\mathrm{Var}(\phi) < \infty$, then the interval from Eq.~\ref{eq:cs} is a valid $(1-\alpha)$-AsympCS.
\end{lemma}

It suffices to verify: (i) the residual error is of smaller order than $\sqrt{\log R / R}$, almost surely; and (ii) $\mathrm{Var}(\phi_v) < \infty$.

\paragraph{Step 1: Second-order remainder analysis}
From the decomposition in Eq.~\ref{eq:decomp}, we need to show $\frac{1}{R}\sum_{r=1}^R m_{t,r} = o_{a.s.}\bigl(\sqrt{\log R / R}\bigr)$, i.e., $\sum_{r=1}^R m_{t,r} = o_{a.s.}\bigl(\sqrt{R\log R}\bigr)$.

\paragraph{Structure of $m_{t,r}$.} For a single arm $a$ and time $t$, the remainder at observation $r$ is:
\begin{equation}
    m_{t, r}^a = \hat{\phi}_t^a(O_r;\, \hat{\eta}_{t, r-1}) - \hat{\phi}_t^a(O_r; \eta_t).
\end{equation}

Expanding using the definition from Eq.~\ref{eq:eif}:
\begin{align}
    m_{t, r}^a  &= \underbrace{\bigl[\hat{S}_{t,r-1}(X_r, a) - S_t(X_r, a)\bigr]}_{\text{term 1}}\label{eq:r-term1}\\
    &- \underbrace{\frac{\mathbf{1}(A_r = a)}{\pi_a(X_r)} \sum_{i=0}^{t} \left[\frac{\hat{S}_{i,r-1}}{\hat{S}_{i}(X_{r-1}, a) \cdot \hat{G}_{i-1}(X_{r-1}, a)}d\hat{M}_i^r- \frac{S_i}{S_j \cdot G_{j-1}}\,dM_i^r\right]}_{\text{term 2}},\label{eq:r-term2}
\end{align}
where we define $d\hat{M}_i^r := \mathbf{1}(\tilde{T}_r = i,\, \Delta_r = 1) - \mathbf{1}(\tilde{T}_r \geq i) \cdot \hat{\lambda}_{i,r-1}^S(X_r, a)$.

\paragraph{Decomposing the summand in~\eqref{eq:r-term1}.} For each $i$, define the weight function $W_i(\eta_t) := S_t(X,a) / [S_i(X,a) \cdot G_{i-1}(X,a)]$ and write:
\begin{align}
    \hat{W}_i \cdot d\hat{M}_i^r - W_i \cdot dM_i^r \nonumber 
    &= \underbrace{(\hat{W}_i - W_i) \cdot dM_i^r}_{\text{(I): weight error $\times$ true martingale}} + \underbrace{\hat{W}_i \cdot (d\hat{M}_i^r - dM_i^r)}_{\text{(II): estimated weight $\times$ hazard error}} \nonumber \\
    &= (\hat{W}_i - W_i) \cdot dM_i^r - \hat{W}_i \cdot \mathbf{1}(\tilde{T}_r \geq i) \cdot (\hat{\lambda}_{i,r-1}^S - \lambda_i^S). \label{eq:two-terms}
\end{align}

\paragraph{Term (I): Weight error.} The cross term involves:
\begin{equation}
    \bigl(\hat{S}_t/\hat{S}_i - S_t/S_i\bigr) \cdot \bigl(1/\hat{G}_{i-1} - 1/G_{i-1}\bigr),
\end{equation}
using $\hat{S}_t/\hat{S}_i = \prod_{k=j+1}^t(1-\hat{\lambda}_k^S)$ and $\hat{G}_{i-1} = \prod_{k=0}^{i-1}(1-\hat{\lambda}_k^G/(1-\hat{\lambda}_k^S))$.
Combining the bounds above yields
\begin{align}
\|\hat W_{i,r-1}(a,\cdot)-W_i(a,\cdot)\|_2
&\le
C_t
\sum_{k=i+1}^{t}
\|\hat\lambda_{k,r-1}^S(a,\cdot)-\lambda_k^S(a,\cdot)\|_2
\nonumber\\
&\quad+
C_t
\sum_{k=0}^{i-1}
\|\hat\lambda_{k,r-1}^S(a,\cdot)-\lambda_k^S(a,\cdot)\|_2
\nonumber\\
&\quad+
C_t
\sum_{k=0}^{i-1}
\|\hat\lambda_{k,r-1}^G(a,\cdot)-\lambda_k^G(a,\cdot)\|_2.
\label{eq:weight-error-L2-direct}
\end{align}
Consequently, after summing over $i=0,\dots,t$, the resulting second-order remainder is bounded by
\begin{align}
&C_t
\left(
\sum_{k=0}^{t}
\|\hat\lambda_{k,r-1}^S(a,\cdot)-\lambda_k^S(a,\cdot)\|_2
\right)^2
\nonumber\\
&\qquad+
C_t
\left(
\sum_{k=0}^{t}
\|\hat\lambda_{k,r-1}^S(a,\cdot)-\lambda_k^S(a,\cdot)\|_2
\right)
\left(
\sum_{k=0}^{t-1}
\|\hat\lambda_{k,r-1}^G(a,\cdot)-\lambda_k^G(a,\cdot)\|_2
\right).
\label{eq:weight-error-remainder-direct}
\end{align}
where $C_{t_{\max}}$ is a constant depending on $t_{\max}$, $\gamma$, and the boundedness of hazards.

\paragraph{Term (II): Hazard estimation error.}
This term involves the event-hazard estimation error $\hat\lambda_{i,r-1}^S(a,\cdot)-\lambda_i^S(a,\cdot)$ directly. As in the standard augmented influence-function argument, the leading linear terms cancel after combining term~(I) and term~(II). In the present tied setting, however, the resulting remainder is not of pure product form in the event and censoring hazard errors, because $G_{i-1}(x,a) = \prod_{k=0}^{i-1} \left(1-\frac{\lambda_k^G(x,a)}{1-\lambda_k^S(x,a)} \right)$, so that perturbations in $\lambda^S$ also induce perturbations in $G_{i-1}$. Consequently, after first-order cancellation, the remaining bias is of second order and is bounded by the sum of a pure event-hazard quadratic term and an event-censoring interaction term.

\paragraph{Aggregating over $i$.}
Using the bounds above for the weight error and the hazard error, and summing over $i=0,\ldots,t$, we obtain that for each $a\in\{0,1\}$,
\begin{align}
|m_{t,r}^a|
&\le
C_t
\left(
\sum_{i=0}^{t}
\|\hat\lambda_{i,r-1}^S(a,\cdot)-\lambda_i^S(a,\cdot)\|_2
\right)^2
\nonumber\\
&\quad+
C_t
\left(
\sum_{i=0}^{t}
\|\hat\lambda_{i,r-1}^S(a,\cdot)-\lambda_i^S(a,\cdot)\|_2
\right)
\left(
\sum_{i=0}^{t-1}
\|\hat\lambda_{i,r-1}^G(a,\cdot)-\lambda_i^G(a,\cdot)\|_2
\right).
\label{eq:mr-bound-app}
\end{align}
Since $t \le t_{\max}$ is fixed, the summation over $i$ contributes only a finite constant and does not alter the rate.

\paragraph{Summing over $r$.}
Suppose the nuisance convergence conditions hold almost surely and $\max_{0\leq i \leq \tmax}\|\hat{\lambda}^S_{i, r-1}(a, \cdot)-\lambda^S_i(a, \cdot) \|_2 = o_{a.s.}\left(\left(\frac{\log r}{r}\right)^{1/4}\right)$, $\max_{0\leq i \leq \tmax-1}\|\hat{\lambda}^G_{i, r-1}(a, \cdot)-\lambda^G_i(a, \cdot) \|_2 = o_{a.s.}\left(\left(\frac{\log r}{r}\right)^{1/4}\right)$. We have
\begin{align}
\label{eq:sum_rate_as}
&\left(
\sum_{i=0}^{t}
\|\hat\lambda_{i,r-1}^S(a,\cdot)-\lambda_i^S(a,\cdot)\|_2
\right)^2
=
o_{a.s.}(\sqrt{\log r/ r}), \\
&\left(
\sum_{i=0}^{t}
\|\hat\lambda_{i,r-1}^S(a,\cdot)-\lambda_i^S(a,\cdot)\|_2
\right)
\left(
\sum_{i=0}^{t-1}
\|\hat\lambda_{i,r-1}^G(a,\cdot)-\lambda_i^G(a,\cdot)\|_2
\right)
=
o_{a.s.}(\sqrt{\log r/ r}),\notag
\end{align}
then Eq.~\ref{eq:mr-bound-app} implies $m_{t,r}=o_{a.s.}(\sqrt{\log r/r})$. Hence,
$\sum_{r=1}^{R}m_{t,r}
=o_{a.s.}\left(\sum_{r=1}^{R}\sqrt{\log r/r}\right)
=o_{a.s.}(\sqrt{R\log R})$,
and therefore
$\frac{1}{R}\sum_{r=1}^{R}m_{t,r}
=o_{a.s.}(\sqrt{\log R/R})$. This establishes the asymptotic linearity condition in Eq.~\ref{eq:al}.

\begin{remark}
In contrast to the no-ties setting, the second-order remainder in the tied setting is not purely of product form in the event and censoring hazard errors. This is because $G_{i-1}(x,a) = \prod_{k=0}^{i-1} \left(1-\frac{\lambda_k^G(x,a)}{1-\lambda_k^S(x,a)} \right)$
so the censoring survival function depends on both $\lambda^G$ and $\lambda^S$. As a result, the remainder contains both $\left(\sum_{i=0}^{t} \|\hat\lambda_{i,r-1}^S(a,\cdot)-\lambda_i^S(a,\cdot)\|_2 \right)^2$ and $\left( \sum_{i=0}^{t} \|\hat\lambda_{i,r-1}^S(a,\cdot)-\lambda_i^S(a,\cdot)\|_2 \right) \left( \sum_{i=0}^{t-1} \|\hat\lambda_{i,r-1}^G(a,\cdot)-\lambda_i^G(a,\cdot)\|_2 \right)$
terms. Since $t\le t_{\max}$ is fixed, summing over $i$ affects only the multiplicative constant and not the asymptotic rate.
\end{remark}

\paragraph{Step 2: Finite variance}
Under Assumptions~\ref{ass:causal} and~\ref{ass:survival}, the influence function is uniformly bounded. Indeed, for arm $a\in\{0,1\}$, we have
\begin{equation}
\phi_t^a(O;\pi,\eta)
=
S_t(X,a)
-
\frac{\mathbf 1(A=a)}{\pi_a(X)}
\sum_{j=0}^{t}
\frac{S_t(X,a)}{S_j(X,a)\,G_{j-1}(X,a)}
\Bigl\{
\mathbf 1(\widetilde T=j,\Delta=1)
-
\mathbf 1(\widetilde T\ge j)\lambda_j^S(X,a)
\Bigr\},
\end{equation}
where $\pi_1(X)=\pi(X)$ and $\pi_0(X)=1-\pi(X)$. By positivity, $\pi_a(X)\ge \epsilon$.
Moreover, for each $j\le t$, $\left|
\mathbf 1(\widetilde T=j,\Delta=1)
-
\mathbf 1(\widetilde T\ge j)\lambda_j^S(X,a)
\right|
\le 1$,
since $\mathbf 1(\widetilde T=j,\Delta=1)\in\{0,1\}$, $\mathbf 1(\widetilde T\ge j)\in\{0,1\}$, and $\lambda_j^S(X,a)\in[0,1]$. Also, $0\le \frac{S_t(X,a)}{S_j(X,a)} \le 1$ for all $j\le t$.
Hence,
\begin{align}
|\phi_t^a(O;\pi,\eta)|
&\le
|S_t(X,a)|
+
\frac{1}{\epsilon}
\sum_{j=0}^{t}
\frac{S_t(X,a)}{S_j(X,a)\,G_{j-1}(X,a)}
\left|
\mathbf 1(\widetilde T=j,\Delta=1)
-
\mathbf 1(\widetilde T\ge j)\lambda_j^S(X,a)
\right|
\nonumber\\
&\le
1+
\frac{1}{\epsilon}\sum_{j=0}^{t}\frac{1}{G_{j-1}(X,a)}.
\label{eq:phi-bound-step}
\end{align}
Under Assumption~\ref{ass:uniform_overlap}, for each $a\in\{0,1\}$ and each $k\le t_{\max}$, we have
$
1-\lambda_k^S(X,a)\ge \underline c_t^S
$
and
$
1-\dfrac{\lambda_k^G(X,a)}{1-\lambda_k^S(X,a)}\ge \underline c_t^G.
$
Hence, for every $j\le t_{\max}$,
$
G_{j-1}(X,a)
=
\prod_{k=0}^{j-1}
\left(
1-\frac{\lambda_k^G(X,a)}{1-\lambda_k^S(X,a)}
\right)
\ge
(\underline c_t^G)^j
\ge
(\underline c_t^G)^{t_{\max}}.
$
Therefore,
$
\frac{1}{G_{j-1}(X,a)}
\le
(\underline c_t^G)^{-t_{\max}}$ for all $j\le \tmax$.
Combining this with Eq.~\ref{eq:phi-bound-step}, we obtain
\begin{align}
|\phi_t^a(O;\pi,\eta)|
&\le
1+\frac{1}{\epsilon}(t_{\max}+1)(\underline c_t^G)^{-t_{\max}}
\nonumber\\
&=:B(\epsilon,\underline c_t^G,t_{\max})<\infty.
\label{eq:bounded}
\end{align}
Consequently,
$
\mathbb E\!\left[\bigl(\phi_t^a(O;\pi,\eta)\bigr)^2\right]
\le
B(\epsilon,\underline c_t^G,t_{\max})^2
<\infty.
$
If, more generally,
$
\phi_v(O;\pi,\eta):=\sum_{t=0}^{t_{\max}} v_t\,\phi_t(O;\pi,\eta),
$
then, by the triangle inequality, we yield
$
|\phi_v(O;\pi,\eta)|
\le
\sum_{t=0}^{t_{\max}} |v_t|\,|\phi_t(O;\pi,\eta)|
\le
\|v\|_1\,B(\epsilon,\underline c_t^G,t_{\max}),
$
and, hence,
$
\Var(\phi_v)
\le
\mathbb E[\phi_v^2]
\le
\|v\|_1^2\,B(\epsilon,\underline c_t^G,t_{\max})^2
<\infty.
$
\end{proof}

\newpage
\section{Extension to Batch Settings}
\label{app:batch}

The theoretical analysis in Section~\ref{sec:theory} is stated for the fully sequential setting in which the policy $\pi_r(X \mid \mathcal{H}_{r-1})$ and the nuisance estimates $\hat{\eta}_{t, r-1}$ are updated after every individual observation. In practice, however, it is common to update both the policy and the nuisance estimates only at the end of each \emph{batch} of observations, thereby reducing computational overhead and stabilizing estimation. We show here that ASE extends naturally to this setting without any modification to the asymptotic normality result.

\textbf{Batch setting:} Partition the $R$ rounds into $B$ batches of size $m = R/B$ each. Let batch $b$ consist of rounds $r \in \mathcal{B}_b := \{(b-1)m + 1, \ldots, bm\}$. All units within batch $b$ are assigned using the same policy $\pi_b(X) := \pi_{(b-1)m+1}(X \mid \mathcal{H}_{(b-1)m})$, computed once from the accumulated history at the end of batch $b-1$. Similarly, the nuisance estimators $\hat{\eta}_{b-1}$ are refitted once per batch using $\mathcal{H}_{(b-1)m}$ via the sequential cross-fitting scheme of Algorithm~\ref{alg:adaptive_survival}. The ASE estimator is $\hat{\tau}_R^{\mathrm{ASE}} = \frac{1}{R}\sum_{r=1}^R \hat{\phi}_r^{\mathrm{cf}}$, as before.

\textbf{The martingale difference sequence (MDS):} The key structural requirement for asymptotic normality is that the oracle pseudo-outcome at horizon $t$,  $\phi(X_{t, r}, Z_{t, r}, A_{t, r}, Y_{t, r}; \pi_{t, r}, \eta)$ forms a martingale difference sequence (MDS) with respect to the filtration $\{\mathcal{H}_{t, r}\}$. In the batch setting, $\pi_b$ is computed once from $\mathcal{H}_{(b-1)m}$ before any unit in batch $b$ is observed. Since $(b-1)m \leq r-1$ for all $r \in \mathcal{B}_b$, the policy $\pi_b$ is already determined by time $r-1$. Hence, the MDS property $\mathbb{E}[\phi(X_{t, r}, Z_{t, r}, A_{t, r}, Y_{t, r}; \pi_b, \eta) \mid \mathcal{H}_{r-1}] = \tau$ continues to hold exactly. The martingale CLT~\citep{cook.2024} then applies without modification to the oracle term in the near-martingale decomposition 
\begin{align}
\sqrt{R}\left(\hat{\tau}_{t, R}^{\mathrm{ASE}} - \tau_t\right) = \frac{1}{\sqrt{R}}\sum_{r=1}^R z_{t, r} + \frac{1}{\sqrt{R}}\sum_{r=1}^R m_{t, r},
\end{align} 
where $z_{t, r} = \phi(\gH_r; \pi_{b}, \eta_{t, r}) - \tau$ is the oracle MDS term and $m_{t, r} = \hat{\phi}_{t, r}^{\mathrm{cf}} - \phi(\gH_r; \pi_{b}, \eta_{t, r})$ is the nuisance remainder.

\textbf{$L_2$-consistency condition in the batch setting.} Theorem~\ref{thm:adaptive_asymptotic_normality_cf} requires $k_r \|\pi_r - \pi^\star\|_2 = o_p(1)$. In the batch setting, $\pi_r \equiv \pi_b$ for all $t \in \mathcal{B}_b$, so this reduces to $k_b \|\pi_b - \pi^\star\|_2 = o_p(1)$ as $b \to \infty$, where $k_b$ is the truncation parameter for batch $b$. This is no harder to satisfy than in the sequential case; indeed, because each batch-level policy update uses $m$ additional observations relative to the previous update, the plug-in estimate of $\pi^\star$ is typically more stable, making the convergence condition easier to verify in practice.

\textbf{Cross-fitting within batches.} Since all $m$ units in batch $b$ share the nuisance estimator $\hat{\eta}_{b-1}$ fitted on $\mathcal{H}_{(b-1)m}$, the standard sequential cross-fitting scheme from Algorithm~\ref{alg:adaptive_survival} applies directly: the history $\mathcal{H}_{(b-1)m}$ is split into two temporal folds, nuisances are fitted on one fold and pseudo-outcomes are evaluated on the other. This ensures that the remainder term satisfies $\frac{1}{\sqrt{R}}\sum_{r=1}^R m_r = o_p(1)$ under the same $L_2$-consistency conditions as in Theorem~\ref{thm:adaptive_asymptotic_normality_cf}, with no additional assumptions.

\textbf{Asymptotic normality is unchanged.} Combining the above, the limiting distribution of the ASE estimator under the batch protocol is identical to the fully sequential case: \begin{align}\sqrt{R}\left(\hat{\tau}_{t, R}^{\mathrm{ASE}} - \tau_t\right) \xrightarrow{d} \mathcal{N}\!\left(0,\, V_{\mathrm{eff}, t}(\pi^\star)\right),\end{align} where $V_{\mathrm{eff}}(\pi^\star)$ is the semiparametric efficiency bound from Theorem~\ref{thm:eff_bound}. In particular, if $\pi_b \to \pi^\star$ as $b \to \infty$, the estimator remains semiparametrically efficient. The batch size $m$ affects only the finite-sample rate at which $\pi_b$ converges to $\pi^\star$ and the computational cost of nuisance estimation, but does not alter the asymptotic variance or the form of valid confidence intervals.

\newpage
\section{Additional experiment results}

\subsection{Twins data}
\label{app:twins_result}

The dataset considers the birth weight of 11984 pairs of twins born in the USA between 1989 and 1991 with respect to mortality in the first year of life. Treatment $a = 1$ corresponds to being born the heavier twin. The dataset contains 46 confounders. For a detailed description of the dataset, see \cite{louizos.2017}.

\begin{figure}[htbp]
    \centering
    \includegraphics[width=1\linewidth]{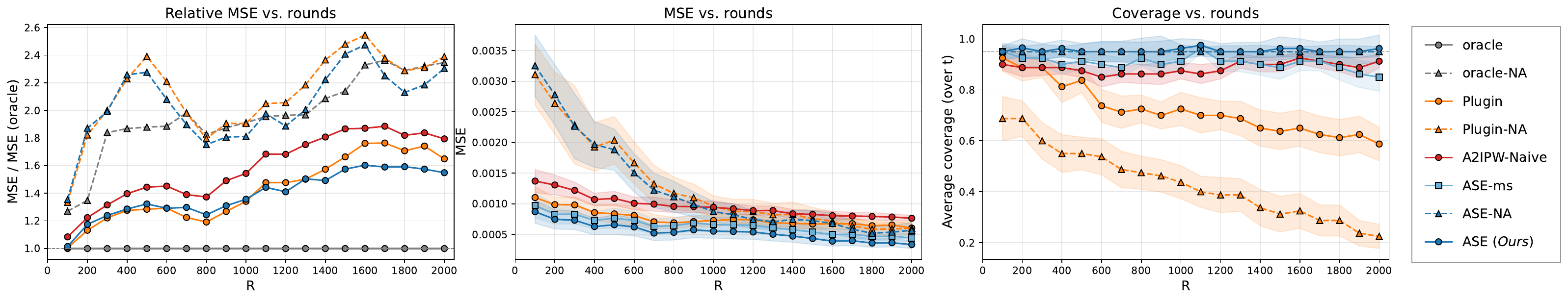}
    \vspace{-0.5cm}
    \caption{\textbf{Twins data:} Adaptive allocation improves estimation efficiency, accuracy, and inferential validity on the semi-synthetic Twins dataset. \textbf{Left:} relative MSE with respect to the Oracle; ASE remains closest to oracle performance among all feasible estimators. \textbf{Middle:} MSE across rounds; ASE converges fastest and ASE-MS remains consistent under nuisance misspecification. \textbf{Right:} empirical coverage of nominal $95\%$ confidence intervals; ASE and ASE-MS maintain nominal coverage while Plugin and A2IPW-NA\"{i}ve deteriorate as $R$ grows.}
    \label{fig:twins_exp_results}
    \vspace{-0.2cm}
\end{figure}

\circledblue{a}~\emph{How much do we benefit from adaptive allocation?}~$\Rightarrow$~\textbf{Effectiveness of censoring-aware A-optimal allocation}: As shown in Fig.~\ref{fig:twins_exp_results}~(a), adaptive design consistently delivers efficiency gains over non-adaptive alternatives. ASE achieves a relative MSE of approximately $1.5\times$ the oracle at $R=2000$, nearly closing the gap to oracle performance despite relying on estimated nuisances. The non-adaptive baselines ASE-NA and Plugin-NA stabilize at roughly $2.3\times$ the oracle MSE, maintaining a persistent efficiency gap throughout. A2IPW-NA\"{i}ve occupies an intermediate position at approximately $1.70\times$--$1.80\times$, confirming that censoring-aware A-optimal allocation, and not just adaptive assignment alone, is the key driver of efficiency gains over uniform randomization.

\circledblue{b}~\emph{Does ASE improve estimation accuracy over rounds?}~$\Rightarrow$~\textbf{Estimation accuracy}: We plot MSE against $R$ in Fig.~\ref{fig:twins_exp_results}~(b). ASE, ASE-MS, and Plugin converge toward the true $\tau$ as $R$ grows, with ASE achieving the lowest MSE due to its variance- and censoring-aware allocation. The results further highlight the importance of \emph{robustness}: ASE-MS tracks ASE closely despite one nuisance component being misspecified, consistent with our theoretical guarantees.

\circledblue{c}~\emph{Is ASE inferentially valid?}~$\Rightarrow$~\textbf{Inferential validity}: We report empirical coverage of nominal $95\%$ confidence intervals in Fig.~\ref{fig:twins_exp_results}~(c). Consistent with our theoretical guarantees, ASE and ASE-MS are the only estimators that maintain nominal coverage uniformly across rounds. In contrast, Plugin, Plugin-NA, and A2IPW-NA\"{i}ve all degraded substantially as $R$ grows, owing to finite-sample bias and censoring bias, respectively.

\subsection{Survival curves}
\label{app:survival_curves}
According to Eq.~\ref{eq:eif}, we can write the non-centered EIF for APO via
\begin{align}
    \phi_a(\mathcal{O};\eta_t) =& S_t(X,a) + \frac{\mathbf{1}(A=a)}{\pi_a(X)}\,\xi(\mathcal{O},\eta_t)\\
    =&  S_t(X,a) + \frac{\mathbf{1}(A=a)}{\pi_a(X)}\sum_{i=0}^t\frac{\mathbf 1(\widetilde T=i,\Delta=1)-\mathbf 1(\widetilde T\ge i)\,\lambda_i^S(X,A)}{S_i(X,A)\,G_{i-1}(X,A)}
\end{align}

In the following, we show the average survival estimation of ASE over rounds.
\vspace{-0.2cm}
\begin{figure}[htbp]
    \centering
    \begin{subfigure}[b]{0.48\linewidth}
        \includegraphics[width=\linewidth]{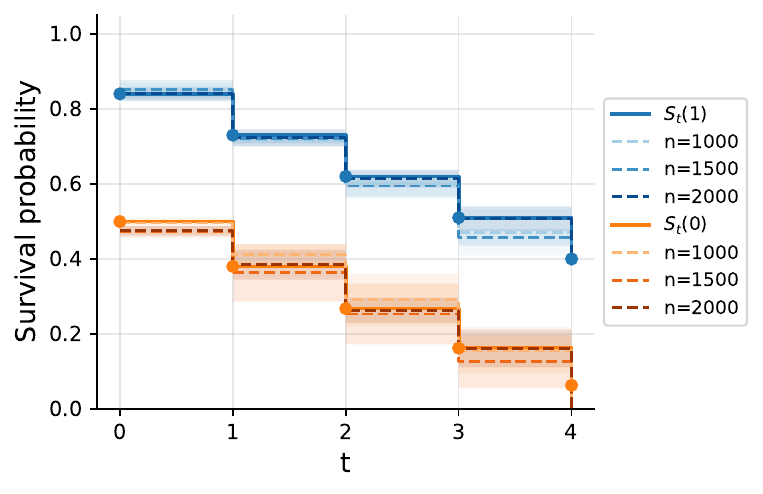}
        \caption{Synthetic data}
    \end{subfigure}
    \hfill
    \begin{subfigure}[b]{0.48\linewidth}
        \includegraphics[width=\linewidth]{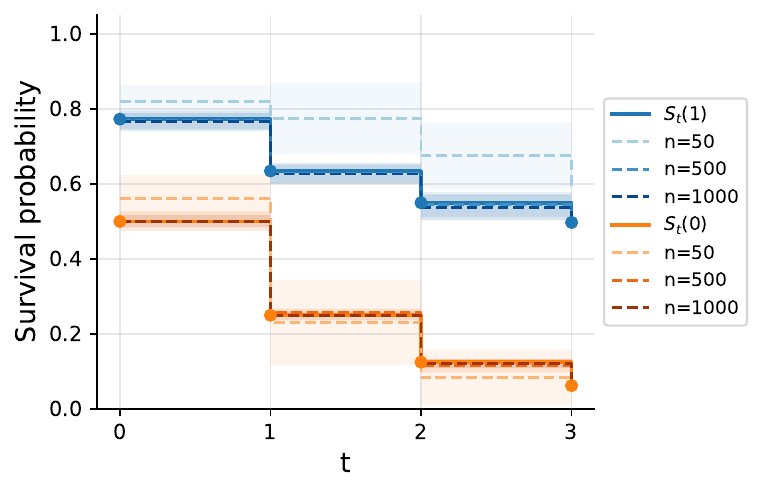}
        \caption{Twins data}
    \end{subfigure}
    \vspace{-0.3cm}
    \caption{Estimated average survival curves $\hat{S}_t(a)$ for treatment ($a=1$, blue) and control ($a=0$, orange) produced by ASE at increasing sample sizes (dashed lines), compared against ground-truth oracle curves $S_t(a)$ (solid lines with dots). Shaded bands denote $\pm 1$ standard error across repeated trials. (\textit{a}) Synthetic data, $t\in\{0,\ldots,4\}$, $n\in\{1000,1500,2000\}$. (\textit{b}) Semi-synthetic Twins data, $t\in\{0,\ldots,3\}$, $n\in\{50,500,1000\}$. Ground-truth construction is detailed in Appendix~\ref{app:implementation_details}.}
    \label{fig:survival_curve}
\end{figure}

Figure~\ref{fig:survival_curve} plots the ASE-estimated average survival curves $\hat{S}_t(a)$ for $a\in\{0,1\}$ across $t\in\{0,\ldots,t_{\max}\}$, compared against the ground-truth oracle curves $S_t(a)$ (solid lines). Dashed lines and shaded bands report the mean and $\pm 1$ standard error of ASE across repeated trials at increasing sample sizes. On the synthetic dataset, ASE recovers the survival curves of both arms accurately even at $n=1000$, with confidence bands tightening visibly as $n$ grows. On the semi-synthetic Twins dataset, a modest upward bias in $\hat{S}_t(1)$ is present at $n=50$ but resolves by $n=1000$, consistent with the asymptotic guarantees of Theorem~\ref{thm:adaptive_asymptotic_normality_cf}.

\FloatBarrier
\clearpage
\section{Implementation details}
\label{app:implementation_details}

\subsection{Data generation}

\textbf{Synthetic data generation:} We sample a one-dimensional confounder $X\sim\mathrm{Uniform}(0,1)$ and set $t_{\max}=4$. Survival and censoring hazards are generated via logistic link functions:
\begin{align}
\lambda^S_t(x,a) &= \sigma\!\left(\alpha_t + \eta(x) + \tau(x)\cdot a\right) \\
\lambda^G_t(x,a) &= \sigma\!\left(\gamma_t + \phi(x) + \psi(x)\cdot a\right) ,
\end{align}
where $\{\alpha_t\}_{t=0}^{4}$ is a time-varying intercept path constructed so that the marginal control survival $S_t(x,0)$ decreases from $0.50$ to $0.02$ over the horizon, and $\{\gamma_t\}_{t=0}^{4}$ is similarly constructed to decrease from $0.84$ to $0.62$. The covariate functions are:
\begin{align}
\eta(x) &= 1.05 + 0.12(x - 0.5) , \\
\tau(x) &= -0.28 + 0.65\cdot\mathbf{1}(x \leq 0.35) - 0.42\cdot\mathbf{1}(x \geq 0.75) , \\
\phi(x) &= 0.06(x - 0.5) , \\
\psi(x) &= 0.14 - 0.26\cdot\mathbf{1}(x \leq 0.35) + 0.20\cdot\mathbf{1}(x \geq 0.75) .
\end{align}
The function $\tau(x)$ captures heterogeneous treatment effect on the event hazard, with the treatment arm exhibiting higher survival probability than control across all $x$. The censoring functions $\phi(x)$ and $\psi(x)$ introduce arm-dependent censoring ($\lambda^G_t(x,0)\neq\lambda^G_t(x,1)$), consistent with the motivating setting of Section~\ref{sec:adaptive_assignment}. Ground-truth average potential outcomes $\mu_a = \mathbb{E}_X[S_t(X,a)]$ are computed via Gauss--Legendre quadrature over $x\in[0,1]$.

\textbf{Twins semi-synthetic data generation:} For our semi-synthetic case study, we employ the Twins dataset following \citet{louizos.2017}, which records birth outcomes for twin pairs in the United States. Following \citet{louizos.2017}, we use the gestational age variable \texttt{GESTAT10} as the single covariate, binarized as $X = \mathbf{1}(\texttt{GESTAT10} > 3)$, which is highly correlated with survival outcome. The time horizon is $t\in\{0,\ldots,3\}$. Survival and censoring hazards are time-homogeneous and defined as:
\begin{align}
\lambda^S_t(x,a) &= \begin{cases} 0.50 , & \text{if } a=0 , \\ 0.40 \cdot \mathbf{1}(x=0) + 0.01 \cdot \mathbf{1}(x=1) , & \text{if } a=1 , \end{cases}\\[6pt]
\lambda^G_t(x,a) &= \begin{cases} 0.05 , & \text{if } a=0 , \\ 0.108 ,  & \text{if } a=1 . \end{cases}
\end{align}
This design exhibits arm-dependent censoring (i.e., $\lambda^G_t(x,0) \neq \lambda^G_t(x,1)$) and covariate-dependent treatment effect heterogeneity in the survival hazard, making it a realistic testbed for our adaptive allocation procedure.

\subsection{Experiment setup}
For both synthetic and semi-synthetic datasets, each estimator was evaluated over $50$ seeds per dataset. Simulations were run over $R=20$ rounds with a burn-in period of $N_0=1000$ rounds, after which nuisance estimators were updated in mini-batches of $100$ units over $20$ successive rounds. For all adaptive methods, we applied the clipped optimal allocation policy from Proposition~\ref{prop:opt_assignment} with truncation parameter $\alpha=0.05$. Oracle methods used ground-truth nuisance functions computed analytically from the DGP parameters, while the misspecified variant \textbf{ASE-MS} was constructed by replacing $\lambda^G(X, A)$ with a constant fit to the marginal average.

For each estimator, we report the mean squared error at round $r$, averaged over time horizons:
\begin{align}
\mathrm{MSE}(r) := \frac{1}{\tmax +1} \sum_{t=0}^{\tmax} \left(\hat{\tau}_t(r) - \tau_t\right)^2.
\end{align}
The relative MSE is computed as $\mathrm{MSE}(r) / \mathrm{MSE}_{\mathrm{Oracle}}(r)$, where $\mathrm{MSE}_{\mathrm{Oracle}}(r)$ is the MSE of the oracle estimator using ground-truth nuisance functions.

\textbf{Inference coverage calculation:} For each horizon \(t\), we construct a nominal \(95\%\) confidence interval for the average survival effect \(\tau_t=\mathbb E\{S_t(X,1)-S_t(X,0)\}\). The reported coverage is averaged over horizons and simulation seeds. Hence, it measures marginal coverage of each \(\tau_t\), consistent with our A-optimal trace criterion, rather than simultaneous coverage of the full survival-effect vector.

\subsection{Implementation}
All experiments were run on an Intel Xeon Silver 4316 CPU with 40 cores and 976GB RAM.

\textbf{Model architecture and parameters:} We estimate the nuisance survival and censoring hazards using LightGBM classifiers, fit separately for each time point and treatment arm, with fixed default hyperparameters.

\end{document}

%% file: math_commands.tex
\usepackage{amsmath,amsfonts,bm}

\def\eqref#1{equation~\ref{#1}}

\def\1{\bm{1}}

\def\vtau{{\bm{\tau}}}

\def\evtau{{\tau}}

\DeclareMathAlphabet{\mathsfit}{\encodingdefault}{\sfdefault}{m}{sl}
\SetMathAlphabet{\mathsfit}{bold}{\encodingdefault}{\sfdefault}{bx}{n}

\def\gA{{\mathcal{A}}}

\def\gH{{\mathcal{H}}}

\def\gO{{\mathcal{O}}}

\def\gT{{\mathcal{T}}}

\def\gX{{\mathcal{X}}}

\newcommand{\E}{\mathbb{E}}

\newcommand{\Var}{\mathrm{Var}}

\newcommand{\Cov}{\mathrm{Cov}}

\newcommand{\diff}{\mathrm{d}}